%% file: arxiv_asgd.tex
\documentclass{article}

\usepackage[preprint]{neurips_2026}

\usepackage[utf8]{inputenc} 
\usepackage[T1]{fontenc}    
\usepackage[pagebackref=true]{hyperref}       
\renewcommand*{\backref}[1]{}
\renewcommand*{\backrefalt}[4]{%
    \ifcase #1 \relax
    \or
        (Cited on page #2.)%
    \else
        (Cited on pages #2.)%
    \fi
}
\usepackage{url}            
\usepackage{booktabs}       
\usepackage{amsfonts}       
\usepackage{nicefrac}       
\usepackage{microtype}      
\usepackage{xcolor}         

\usepackage{subcaption}
\usepackage{amsmath}
\allowdisplaybreaks

\usepackage{etoolbox}
\newtoggle{preprint}
\togglefalse{preprint} 

\newcommand{\textstyleifpreprint}{\iftoggle{preprint}{\textstyle}{}}

\title{
  Rescaled Asynchronous SGD:\\
  Optimal Distributed Optimization\\
  under Data and System Heterogeneity
}
\author{%
  Ammar~Mahran \\
  KAUST \\
  \texttt{majiedammar.mahran@kaust.edu.sa} \\
  \And
  Artavazd~Maranjyan \\
  KAUST \\
  \texttt{artavazd.maranjyan@kaust.edu.sa} \\
  \And
  Peter~Richt\'{a}rik \\
  KAUST \\
  \texttt{peter.richtarik@kaust.edu.sa} \\
}

\input{macros.tex}

\usepackage{thmtools}
\usepackage{thm-restate}
\usepackage{cleveref}
\crefname{assumption}{assumption}{assumptions}
\crefname{theorem}{theorem}{theorems}
\crefname{corollary}{corollary}{corollaries}
\crefname{lemma}{lemma}{lemmas}

\newcommand{\algcomment}[1]{\hfill $\triangleright$ #1}

\newcommand{\var}[2][]{\mathbb{V}{#1}\rbr{#2}}
\newcommand{\cov}[2][]{\mathbb{C}{#1}\rbr{#2}}

\newcommand{\rasgd}{\algname{Rescaled~ASGD}}
\newcommand{\rasgdtitle}{Rescaled~ASGD}
\newcommand{\vasgd}{\algname{Vanilla~ASGD}}

\newcommand{\mishchenkoasgd}{\algname{Delay-Adaptive~ASGD}}
\newcommand{\koloskovaasgd}{\algname{Concurrent~ASGD}}

\definecolor{linkcolor}{rgb}{0.45,0,0.05} 
\hypersetup{ %
  colorlinks=true,
  linkcolor=linkcolor,
  citecolor=linkcolor,
  filecolor=linkcolor,
  urlcolor=linkcolor,
}

\begin{document}

\maketitle

\begin{abstract}
  Asynchronous stochastic gradient descent (ASGD) is a standard way to exploit heterogeneous compute resources in distributed learning:
  instead of forcing fast workers to wait for slow ones, the server updates the model whenever a gradient arrives.
  \vasgd applies each arriving gradient with the same weight.
  When local data distributions are heterogeneous, this becomes problematic: faster workers contribute more updates, and we show theoretically that the method is biased toward a frequency-weighted average of the local objectives rather than the desired global objective.
  Existing remedies typically move away from the simple ASGD template by introducing gathering phases, buffering, or extra memory.
  We show that this is unnecessary.
  Keeping the standard ASGD mechanism, we recover the correct objective by rescaling worker-specific stepsizes in proportion to their computation times, so that each worker contributes the same aggregate learning rate over a cycle.
  In the non-convex setting, under smoothness and bounded heterogeneity assumptions, we prove that the resulting method, \rasgd, converges to stationary points of the correct global objective in the fixed-computation model.
  Its time complexity matches the known lower bound in the leading term, while the effects of staleness and data heterogeneity appear only in lower-order terms.
  Experiments confirm that the method converges to the correct objective and is competitive with state-of-the-art baselines.
\end{abstract}

\section{Introduction} 
\label{sec:introduction}

We consider the non-convex distributed optimization problem
\begin{equation}
\label{eq:objective-intro}
  \textstyleifpreprint
    \min_{\xx \in \R^d} \nbr{F(\xx) \coloneqq \frac{1}{n} \sum_{i=1}^n F_i(\xx)} ,
\end{equation}

where $F_i : \R^d \to \R, \xx \mapsto \ExpSub{\xi \sim \cD_i}{f_i(\xx; \xi)}$ denotes the local objective function of worker $i = 1, \dots, n$, which is the expected value of the sample loss $f_i(\xx; \xi)$ over data points $\xi$ drawn from their local distribution $\cD_i$.
In parallel stochastic gradient descent methods, workers collaboratively solve this problem by computing stochastic gradients of their local objectives, which are then processed by a central server to find an $\varepsilon$-stationary point, that is, a random vector $\bar\xx \in \R^d$ such that $\mathbb{E}\norm{\nabla F(\bar\xx)}^2 \le \varepsilon$.

In this model, workers differ across two dimensions.
First, each worker has their own local objective function $F_i$ which may differ from one another and the global objective function $F$.
These objective functions vary across workers, typically due to highly heterogeneous local data distributions $\cD_i$, as is common in federated learning \citep{mcmahan_communication-efficient_2017,caldas_leaf_2019,kairouz_advances_2021,wang_field_2021}, for example.
While early asynchronous literature primarily focused on the homogeneous data setting \citep[e.g.,][]{tyurin_optimal_2023,maranjyan_ringmaster_2025}, recent extensions to the heterogeneous setting \citep{mishchenko_asynchronous_2022,koloskova_sharper_2022, nguyen_federated_2022} often require restrictive similarity assumptions or algorithmic modifications.
In the present work, our focus is on the general case allowing for data heterogeneity ($F_i \neq F_j$) under a relaxed similarity assumption (\Cref{assumption:bounded-heterogeneity}).

Second, workers may also differ in the time it takes them to compute stochastic gradients, due to, for example, hardware differences \citep{dutta_slow_2018,li_federated_2020-1,maranjyan_mindflayer_2025, maranjyan_ata_2025, maranjyan_gradskip_2025}.
We denote the time needed for worker $i$ to compute a stochastic gradient by $\tau_i$ throughout the text and explicitly allow for these to differ across workers.

\begin{table*}[t]
    \caption{
    Comparison of worst-case wall-clock time complexities for parallel stochastic first-order methods in the fixed-computation time model with data and system heterogeneity to achieve $\varepsilon$-stationarity in \eqref{eq:objective-intro}.
    We denote the arithmetic mean and maximum, respectively, of the workers' computation times by $\tau_A, \tau_{\max}$.
    Problem parameters include the initial optimality gap $\Delta \coloneqq F(\xx^0) - F^*$ and global smoothness constant $L$ (\Cref{assumption:global-function-nc}), the target stationarity $\varepsilon > 0$, and a bound on the stochastic gradient variance $\sigma^2$ (\Cref{assumption:bounded-gradient-variance}).
    \textbf{Asymptotic Optimality:} time complexity matches the lower bound established by \citet{tyurin_optimal_2023} in the leading term and achieves $\varepsilon$-stationarity for arbitrarily small $\varepsilon$.
    \textbf{No Idle Workers:} all workers remain busy and computational resources are fully utilized.
    \textbf{No Memory Overhead:} no gradients or model iterates need to be stored for later use.
    }
    \label{table:complexities}
    \centering 
    \begin{threeparttable}
    \resizebox{\textwidth}{!}{
    \begin{tabular}[t]{c c c c c}
      \toprule
      \bf \makecell{Method \\ (Reference)} & \bf \makecell{Time Complexity \\ (Leading Term)} & \bf \makecell{Asymptotic \\ Optimality} & \bf \makecell{No Idle \\ Workers} & \bf \makecell{No Memory \\ Overhead} \\
      \midrule
      \makecell{\naiveminibatch \\ \citep{cotter_better_2011,dekel_optimal_2012}} & $\frac{\Delta L \sigma^2}{n \varepsilon^2} \tau_{\max}$& \crossmarkred & \crossmarkred  & \checkmarkgreen \\
      \addlinespace
      \makecell{\malenia \\ \citep{tyurin_optimal_2023}} & $\frac{\Delta L \sigma^2}{n \varepsilon^2} \tau_A$ & \checkmarkgreen & \checkmarkgreen & \crossmarkred \\
      \addlinespace
      \makecell{\ringleader \\ \citep{maranjyan_ringleader_2026}} & $\frac{\Delta L' \sigma^2}{n \varepsilon^2} \tau_A$ \textsuperscript{{\color{linkcolor} (\dag)}} & \crossmarkred & \checkmarkgreen & \crossmarkred \\
      \addlinespace
      \makecell{\koloskovaasgd \\ \citep{koloskova_sharper_2022}} & $\frac{\Delta L (\sigma^2 + \zeta^2)}{n \varepsilon^2} \tau_{\max}$ \textsuperscript{{\color{linkcolor} (\ddag)}} & \crossmarkred & \crossmarkred & \crossmarkred \\
      \addlinespace
      \makecell{\mishchenkoasgd \\ \citep{mishchenko_asynchronous_2022}} & $\frac{\Delta L \sigma^2}{n \varepsilon^2} \tau_A$ & \crossmarkred \textsuperscript{{\color{linkcolor} (\ddag)}} & \checkmarkgreen & \checkmarkgreen \\
      \addlinespace
      \makecell{\rasgd \\ \Cref{thm:r-asgd-bound}, \Cref{corollary:r-asgd-time-complexity}} & $\frac{\Delta L \sigma^2}{n \varepsilon^2} \tau_A$ & \checkmarkgreen & \checkmarkgreen & \checkmarkgreen \\
      \bottomrule
    \end{tabular}
    }

    \resizebox{\textwidth}{!}{
        \begin{minipage}{\textwidth}
            \begin{tablenotes}[para,flushleft]
                \footnotesize
                \item[{\color{linkcolor} (\dag)}]
                \citet{maranjyan_ringleader_2026} rely on a constant $L' \ge L$ to bound data heterogeneity. 
                In certain settings, $L' = \cO(L)$, and \ringleader achieves optimal time complexity.
                \\
                \item[{\color{linkcolor} (\ddag)}]
                \citet{koloskova_sharper_2022,mishchenko_asynchronous_2022} rely on a stricter assumption on the data heterogeneity, $\norm{\nabla F_i(\xx) - \nabla F(\xx)}^2 \le \zeta^2$. 
                In \koloskovaasgd, the constant $\zeta^2$ affects the leading term of the time complexity.
                In \mishchenkoasgd, the gradient norm can be bounded only up to this constant, therefore not achieving arbitrary $\varepsilon$-stationarity, and thus no asymptotic optimality.
                If the computation times are equal across workers, \mishchenkoasgd achieves arbitrary $\varepsilon$-stationarity with the leading term of the time complexity scaling with $\tau_A$.
                See \Cref{sec:wall_clock_time_complexities_of_concurrent_and_delay_adaptive_asgd} for details.
            \end{tablenotes}
        \end{minipage}
    }
    \end{threeparttable}
\end{table*}

\paragraph{Naive Distributed Optimization} 
\label{par:naive_minibatch_sgd}

In \naiveminibatch \citep{cotter_better_2011,dekel_optimal_2012}, each worker $i$ computes a stochastic gradient in $\tau_i$ units of time and sends it to the central server.
Once the server has received gradients from all workers, gradients are averaged, and the model is updated and sent back to the workers for computation of another round of gradients.
Each round thus takes $\tau_{\max} \coloneqq \max_i \tau_i$ units of time, leaving faster workers idle after delivering their gradients while they wait for the global model update.
While this fully synchronized procedure simplifies the theoretical analysis to that of standard SGD, it severely underutilizes fast workers whose idle time could be spent more productively.


\paragraph{Leveraging Faster Workers} 
\label{par:leveraging_faster_workers}

\citet{tyurin_optimal_2023} introduce an improved variant, \malenia, that addresses worker idleness.
In this algorithm, workers compute a number of stochastic gradients at the current model until the server receives at least one stochastic gradient from each.
Thus, faster workers do not idly wait after computation of a single stochastic gradient, but continue computing stochastic gradients at the same model iterate.
These gradients are then averaged per worker, so that each worker effectively delivers one stochastic gradient with reduced variance.
Once the server has obtained at least one gradient from each worker, the model is updated and sent back to all workers for the next round.
While the gradient gathering phase is asynchronous, the model update happens synchronously ensuring all workers compute gradients at the same model iterate in each round.

Building on this, \citet{maranjyan_ringleader_2026} introduce an asynchronous algorithm, \ringleader, in which the number of model updates is greater, allowing for greater exploration of the search space.
Here, the gradient gathering phase is followed by $n$ asynchronous updates, one from each worker.
Notably, \ringleader utilizes a gradient table in which gradients from all workers are stored to ensure that once an asynchronous update is taken, the step direction approximates the gradient of the global objective function rather than that of the workers' local objective functions.
As in \naiveminibatch and \malenia, the gathering phase causes the slowest worker to dictate when updates from all workers can be applied.
Nevertheless, analysis shows that the leading term of the wall-clock time complexity of \malenia and \ringleader scales with the arithmetic mean of the computation times, not their maximum.
This speed-up stems from the variance reduction of having faster workers compute multiple stochastic gradients per round.


\begin{algorithm}[t]
\caption{Asynchronous SGD}
\label{algo:asgd}
\begin{algorithmic}[1]
\STATE \textbf{Input:} initial point $\xx_0\in \R^{d}$, stepsizes $\gamma_k > 0$

\STATE Set $\yy_{0,i} = \xx_0, \quad \forall i$
\algcomment{$\yy_{k,i} = $ worker $i$'s model before update $k$.}%

\STATE Each worker $i$ begins calculation of $\nabla f_i(\yy_{0,i})$

\FOR{$k = 0,1,\dots$}
\STATE \textit{Some} worker $i_k$ delivers stochastic gradient $\nabla f_{i_k}(\yy_{k,i_k})$
\STATE Server updates $\xx_{k+1} = \xx_k - \gamma_k \nabla f_{i_k}(\yy_{k,i_k})$
\algcomment{Model update with stale gradient.}%

\STATE Update worker model:
\newlength{\lhswidth}
\settowidth{\lhswidth}{$\yy_{k+1,i_k}$}
\item[] \quad $\makebox[\lhswidth][l]{$\yy_{k+1,i_k}$} = \xx_{k+1}$ 
\algcomment{Update model held by $i_k$.}%
\item[] \quad $\makebox[\lhswidth][l]{$\yy_{k+1,j}$} = \yy_{k,j} , \quad \forall j \neq i_k$ 
\algcomment{Models held by other workers unchanged.}%

\STATE Worker $i_k$ begins calculating $\nabla f_{i_k}(\xx_{k+1})$
\ENDFOR
\end{algorithmic} 
\end{algorithm}

\paragraph{Asynchronous SGD} 
\label{par:fully_asynchronous_sgd}

In contrast to these methods, \asyncsgd never synchronizes workers nor delays the application of updates from any workers.
Instead, the server updates the model whenever it receives a gradient from a worker, and immediately forwards the updated model to the same worker for computation of their next stochastic gradient.
This method is described in \Cref{algo:asgd}.\footnote{Henceforth, \asyncsgd refers to \Cref{algo:asgd} specifically, while "asynchronous SGD" refers to the wider class of asynchronous SGD methods.}

\asyncsgd avoids the problem of idle workers, but introduces two novel problems: \textbf{staleness} and \textbf{objective inconsistency}.
Because the server updates its model whenever a gradient arrives from any worker, the gradient applied in a given iteration typically does not correspond to the model of the previous iteration;
in other words, the server performs model updates with stale gradients.
Moreover, as faster workers contribute more gradient updates than slower ones, the search trajectory may be biased towards faster workers' local objective functions, resulting in the wrong objective function being targeted.

\citet{koloskova_sharper_2022} analyze a modified version of \Cref{algo:asgd}.
Here, workers compute stochastic gradients in parallel and the server updates the model whenever an update arrives, just like in \Cref{algo:asgd}.
Unlike \asyncsgd, the server then selects a worker uniformly at random to compute a stochastic gradient at the updated model.
Crucially, workers are sampled uniformly at random, and even workers who are still busy computing a stochastic gradient may be sampled for computation of another gradient at this updated iterate.
Thus, the total number of gradients in flight is kept constant throughout this algorithm.
While theoretical analysis shows this method to converge to a stationary point, it does not leverage faster workers to compute more gradients as the shares of model updates by all workers are, by construction, equal in the long run.

\citet{mishchenko_asynchronous_2022} offer the latest analysis of \Cref{algo:asgd} under arbitrary delays.
In the heterogeneous data setting, the authors can only establish a bound on the gradient norm with a non-vanishing term capturing the degree of data heterogeneity, i.e., a measure of how much the gradients of the local objective functions differ from that of the global objective function.
As the authors rightly point out, "we cannot really expect good performance for arbitrarily heterogeneous losses under arbitrary delays."


\paragraph{From Gradient to Gradient \textit{Step}} 
\label{par:motivation}

Fortunately, there is often more temporal structure in practice than assumed in the analysis of \citet{mishchenko_asynchronous_2022} that we can exploit.
To a first-order approximation, computation times may, in many settings, be assumed to be fixed over time. 
To see how this added structure can be leveraged, let us first consider what happens when \naiveminibatch takes a step of stepsize $\alpha$:
The model $\xx^0$ is moved in direction $-\alpha \cdot \frac{1}{n} \sum_{i=1}^n \nabla f_i(\xx^0)$, so that the contribution of worker $i$ to this update is $-\nicefrac{\alpha}{n} \nabla f_i(\xx^0)$.
In expectation, this becomes $-\nicefrac{\alpha}{n} \nabla F_i(\xx^0)$.

Next, consider what \asyncsgd can do in the same time, and let us assume, for simplicity, that worker $i$ computes $K_i$ stochastic gradients at points $\yy_0^0, \dots, \yy_{K_i-1}^0$ in the time needed for one \naiveminibatch update.
If each step of worker $i$ has stepsize $\gamma_i$, then the contribution of worker $i$ over this time is $-\sum_{k=0}^{K_i-1} \gamma_i \nabla f_i(\yy_k^0)$.
Provided that the $K_i$ points $\yy_0^0, \dots, \yy_{K_i-1}^0$ are close to $\xx^0$ and the evaluated gradients do not change abruptly, this contribution is approximately equal to $-K_i \gamma_i \nabla F_i(\xx^0)$ in expectation.
Comparison with the contribution made by worker $i$ in \naiveminibatch suggests choosing a worker-specific stepsize $\gamma_i = \nicefrac{\alpha}{n K_i}$ to approximate the ideal gradient step $-\alpha \nabla F(\xx^0)$ through the sum of all workers' contributions.

The sequence of stochastic gradient steps thus computed by \asyncsgd is, in general, not an unbiased estimator of the true gradient step $-\alpha \nabla F(\xx^0)$.
However, our analysis shows that this bias can be controlled efficiently.
Moreover, by letting faster workers take more steps we achieve greater variance reduction and faster convergence.


\paragraph{Contributions and Insights} 
\label{par:contributions}

Our primary conceptual insight is a shift in perspective: 
asynchronous methods do not need to approximate the true global gradient direction in every single iteration. 
Instead, we show that by approximating the global gradient \textit{step}---split across workers and over time---we can establish convergence to an $\varepsilon$-stationary point.

We show that rescaling worker-specific stepsizes (\rasgd) neutralizes objective inconsistency under data and system heterogeneity. 
In the fixed-computation model, \rasgd achieves a near-optimal wall-clock time complexity to reach an $\varepsilon$-stationary point, matching known lower bounds in the leading term (cf. \Cref{table:complexities}). 
As an additional result of our analysis, we show that \vasgd, which employs the same stepsizes for all workers, targets a frequency-weighted average of the local objectives.

\rasgd maintains the standard ASGD template (\Cref{algo:asgd}) without introducing memory overhead, gathering phases, or worker idleness. 
Through proof-of-concept experiments on heterogeneous data (cf. \Cref{sec:experiments}), we confirm that \rasgd accurately targets the global objective and remains competitive with memory- and synchronization-heavy baselines.


\section{Problem Setup} 
\label{sec:problem_setup}

Let $\xi_{k}$ denote the random sample drawn by worker $i_k$ for the stochastic gradient delivered in iteration $k$. 
As is standard, we assume the collection of samples $\{\xi_{k}\}_{k \in \mathbb{N}}$ to be mutually independent.
Moreover, we make the following assumptions about the stochastic gradients:

\begin{assumption}[Unbiased Gradients]
\label{assumption:unbiasedness}
    The stochastic gradients are unbiased, that is
    \begin{equation*}
      \ExpSub{\xi \sim \cD_i}{\nabla f_i(\xx, \xi)} = \nabla F_i(\xx) , \quad \forall \xx \in \R^d, \quad \forall i \in \nbr{1, \dots, n} .
    \end{equation*}
\end{assumption}

\begin{assumption}[Bounded Variance]
\label{assumption:bounded-gradient-variance}
    The stochastic gradients have bounded variance $\sigma^2 \ge 0$, that is,
    \begin{equation*}
    \textstyleifpreprint
      \mathbb{E}_{\xi \sim \cD_i}{\norm{\nabla f_i(\xx, \xi) - \nabla F_i(\xx)}^2} \le \sigma^2 , \quad \forall \xx \in \R^d, \quad \forall i \in \nbr{1, \dots, n} .
    \end{equation*}
\end{assumption}

Recall the standard definition of smoothness.

\begin{definition}[Smoothness]
  A differentiable function $F: \R^d \to \R$ is called $L_F$-smooth if
  \begin{equation*}
    \norm{\nabla F(\xx) - \nabla F(\yy)} \le L_F \norm{\xx-\yy} , \quad \forall \xx, \yy \in \R^d .
  \end{equation*}

  By convention, $L_F$ denotes the smallest such constant.
\end{definition}

For our analysis, we require the functions involved to be smooth.

\begin{assumption}[Global Objective Function]
\label{assumption:global-function-nc}
    The global objective function $F$ is differentiable and $L$-smooth,
    and bounded below by $F^* \coloneqq \inf_{\xx \in \R^d} F(\xx) > - \infty$.
\end{assumption}

We denote the initial optimality gap by $\Delta \coloneqq F(\xx^0) - F^*$.

\begin{assumption}[Local Objective Functions]
\label{assumption:local-f-smooth}
    Each local objective function $F_i$ is differentiable and $L_i$-smooth.
\end{assumption}

We denote by $L_{\max} \coloneqq \max_i L_i$ the maximum of these local smoothness constants.

\begin{assumption}[Bounded Heterogeneity]
\label{assumption:bounded-heterogeneity}
    There exist constants $\zeta, \rho \ge 0$ such that
    \begin{equation*}
      \norm{\nabla F_i(\xx)}^2 \le \zeta^2 + \rho^2 \norm{\nabla F(\xx)}^2 , \quad \forall \xx \in \R^d , \quad \forall i \in \nbr{1, \dots, n} .
    \end{equation*}
\end{assumption}

\Cref{assumption:unbiasedness,assumption:bounded-gradient-variance,assumption:global-function-nc} are standard in first-order methods \citep{ghadimi_stochastic_2013, bubeck_convex_2015, bottou_optimization_2018}.

Limiting function heterogeneity as in \Cref{assumption:bounded-heterogeneity} is necessary to establish convergence for \asyncsgd and similar  methods that take steps based on gradients from a single worker.
While early asynchronous literature often relied on the overly restrictive assumption of bounded gradients ($\rho = 0$) \citep{bertsekas_incremental_2011,recht_hogwild_2011}, recent analyses \citep{mishchenko_asynchronous_2022,koloskova_sharper_2022} relax this but still constrain the gradient scaling. 
Our assumption is even weaker, allowing us to encompass a significantly broader class of functions.
In \Cref{sec:on_the_bounded_data_heterogeneity_assumption}, we offer a deeper discussion on these bounded heterogeneity assumptions.



\section{\asyncsgdtitle~with Cyclic Update Schedule and \rasgdtitle} 
\label{sec:asyncsgd_with_a_cyclic_update_schedule}

Following \citet{mishchenko_asynchronous_2022,tyurin_optimal_2023}, we assume worker $i$ to take a fixed amount of time, $\tau_i$, to compute one stochastic gradient.
These times may differ across workers, but do not vary over time for a given worker.
To facilitate the theoretical analysis and derive the convergence guarantees in \Cref{sec:theory}, we must ensure that workers deliver a fixed number of updates over a given time horizon.
To that end, we make the following assumption:

\begin{assumption}[Harmonic Periods]
\label{assumption:computation-times}
  The set of computation times $\{\tau_1, \dots, \tau_n\}$ satisfies $\nicefrac{\tau_i}{\tau_j} \in \mathbb{N}$ or $\nicefrac{\tau_j}{\tau_i} \in \mathbb{N}$ for all $i,j \in \{1, \dots, n\}$.
\end{assumption}

To the best of our knowledge, this structural assumption is a novelty in the optimization literature and merits some discussion.
First, while it imposes structure on the computation times, it preserves the primary characteristic of interest: worker heterogeneity. 
Second, unlike the functional assumptions introduced in \Cref{sec:problem_setup}, this constraint can be strictly enforced in practice. 
Real-time systems scheduling often uses harmonic task sets to maximize resource utilization in exactly this manner \citep{han1997better,sudvarg2024elastic}. 
For example, a small delay can be added to round the computation times to the next-largest power of two, slowing down any given worker by a factor of two at worst.
Third, this structure renders the theoretical analysis tractable and allows us to derive the interpretable bounds on the wall-clock time complexity presented in \Cref{sec:theory}.

The performance of \rasgd remains robust in practice even when this assumption is relaxed.
In \Cref{sec:experiments}, we demonstrate that our method remains robust in settings with non-harmonic periods and stochastically fluctuating computation times.
This suggests that while our fixed-computation model imposes strict theoretical structure, the conceptual insights generalize to much broader, unstructured settings.


\paragraph{Cyclic Update Schedule} 
\label{par:cyclic_update_schedule}

Under \Cref{assumption:computation-times}, \asyncsgd as presented in \Cref{algo:asgd} becomes more structured.
Assuming, for ease of exposition, that the server processes updates delivered by multiple workers at the same time in a fixed order, the sequence of worker indices $i_k$ is fully deterministic and follows a cyclic update schedule.
After $\tau_{\max}$ time units, each worker will have sent $K_i \coloneqq \frac{\tau_{\max}}{\tau_i}$ updates in a fixed order that will be repeated exactly over the next $\tau_{\max}$ time units.

We refer to one such pass as a \emph{cycle}, to $\tau_{\max}$ as the (wall-clock) cycle duration, and define $K \coloneqq \sum_{i=1}^n K_i$ as the total number of updates received from all workers over the course of one cycle.
We may now reformulate \Cref{algo:asgd} in this more structured form as shown in \Cref{algo:asgd-iv}.
Here, $m$ denotes the cycle and $k$ the iteration within a cycle.

\begin{algorithm}[htbp]
\caption{Asynchronous SGD with Cyclic Update Schedule}
\label{algo:asgd-iv}
\begin{algorithmic}[1]

\STATE \textbf{Input:} initial point $\xx^0\in \R^{d}$, stepsizes $\gamma_k^m > 0$

\STATE Set $\xx_0^0 = \xx^0$
\algcomment{$\xx_k^m = $ server model before update $k$ in cycle $m$.}%
\STATE Set $\yy_{0,i}^0 = \xx^0 , \quad \forall i$
\algcomment{$\yy_{k,i}^m = $ worker $i$'s model before update $k$ in cycle $m$.}%

\STATE Each worker $i$ begins calculation of $\nabla f_i(\yy_{0,i}^0)$

\FOR{$m = 0, 1, 2, \dots$}
    \FOR{$k = 0, \dots, K - 1$}
        \STATE Worker $i_k$ delivers stochastic gradient $\nabla f_{i_k}(\yy_{k,i_k}^m)$
        \STATE Server updates $\xx_{k+1}^m = \xx_k^m - \gamma_k^m \nabla f_{i_k}(\yy_{k,i_k}^m)$
        \algcomment{Model update with stale gradient.}%

        \STATE Update worker model:
        \newlength{\lhswidthcycle}
        \settowidth{\lhswidthcycle}{$\yy_{k+1,i_k}^m$}

        \item[] \quad $\makebox[\lhswidthcycle][l]{$\yy_{k+1,i_k}^m$} = \xx_{k+1}^m$
        \algcomment{Update model held by $i_k$.}%

        \item[] \quad $\makebox[\lhswidthcycle][l]{$\yy_{k+1,j}^m$} = \yy_{k,j}^m , \quad \forall j \neq i_k$ 
        \algcomment{Models held by other workers unchanged.}%

        \STATE Worker $i_k$ begins calculating $\nabla f_{i_k}(\xx_{k+1}^m)$
    \ENDFOR

    \STATE Set $\yy_{0,i}^{m+1} = \yy_{K,i}^m, \quad \forall i$
    \algcomment{Last point seen in current cycle.}%

    \STATE Set $\xx^{m+1} = \xx_K^m$
    \algcomment{End-of-cycle update.}%
    \STATE Set $\xx_0^{m+1} = \xx_K^m$
    \algcomment{First point of new cycle.}%
\ENDFOR

\end{algorithmic}
\end{algorithm}

\Cref{algo:asgd-iv} initializes a shared point $\xx^0$, after which all workers proceed asynchronously without any synchronization or idle periods. 
We refer to $\nbr{\xx^m}_{m \in \N}$ as the cycle iterates and to $\nbr{\xx_k^m}_{k =1, \dots, K}$ as the inner iterates within cycle $m$. 
Note that each cycle comprises $K = \sum_{i=1}^n K_i = \sum_{i=1}^n \nicefrac{\tau_{\max}}{\tau_i} = \frac{n \tau_{\max}}{\tau_H}$ updates, where $\tau_H$ is the harmonic mean of the computation times.

Let $\yy_{k,i}^m$ be the local model held by worker $i$ at the start of inner iteration $k$ in cycle $m$. 
This model corresponds to a past server iterate $\xx_{k'}^{m'}$, originating either from the previous cycle ($m'=m-1$, $k' \ge k$) or the current one ($m'=m$, $k'<k$). 
As a result, the gradient staleness, or delays---that is, the number of server updates between a worker reading the model and applying its gradient---is bounded by $K$.
Bounded staleness of this form is a well-established setting in the asynchronous optimization literature \citep{agarwal_distributed_2011,recht_hogwild_2011,lian_asynchronous_2015}.

By definition, cycle iterates $\xx^m$ capture the server state at the start of cycle $m$.
Thus, each worker $i$ contributes exactly $K_i$ updates between $\xx^m$ and $\xx^{m+1}$.
This additional structure allows us to analyze the sequence of cycle iterates.


\paragraph{Rescaled ASGD} 
\label{par:rescaled_asgd}

We aim to find a stationary point of the objective function \eqref{eq:objective-intro}. 
Our proposed method sets the stepsizes in \Cref{algo:asgd-iv} to $\gamma_k^m = \gamma_{i_k} \propto \tau_{i_k}$. 
By doing so, even though the number of gradients delivered varies across workers, the aggregate stepsizes taken along each worker's descent direction over a cycle, $\gamma_i K_i = \gamma_i \cdot \nicefrac{\tau_{\max}}{\tau_i}$, are equal.
The accumulated gradient update taken by \rasgd over the course of a cycle $m$ can then be decomposed (cf. \Cref{lemma:cycle-step-decomposition}) into the exact scaled global gradient $\alpha \nabla F(\xx^m)$, a bias term $\mathbf{b}_m$, and a noise term $\boldsymbol{\nu}_m$:
  \begin{equation*}
  \textstyleifpreprint \sum_{k=0}^{K-1} \gamma_{i_k} \nabla f_{i_k}(\yy_{k,i_k}^m) = \alpha \nabla F(\xx^m) + \mathbf{b}_m + \boldsymbol{\nu}_m ,
  \end{equation*}
where $\alpha \coloneqq \sum_{k=0}^{K-1} \gamma_{i_k}$ is the cycle stepsize.

The bias term $\mathbf{b}_m$, unlike the noise term $\boldsymbol{\nu}_m$, does not vanish in expectation.
As we will see next, however, the impact of this bias can be efficiently controlled by scaling down the stepsizes.



\section{Convergence Guarantees} 
\label{sec:theory}

Our main result establishes that \rasgd with properly chosen, worker-specific stepsizes targets the equal-weighted average $F$ defined in \eqref{eq:objective-intro}.
For a chosen stepsize parameter $\gamma > 0$, we set the worker-specific stepsizes as
\begin{equation}
  \label{eq:learning-rates-equal}\textstyleifpreprint
  \gamma_i \coloneqq \gamma \cdot \tau_i \cdot \frac{\tau_H}{n \tau_{\max}} ,
\end{equation}
where the last factor serves as a normalizing constant. 
Under this rescaling, we obtain the following convergence guarantee, the proof of which, as well as those of all other results in this section, is deferred to \Cref{sec:proofs}.

\begin{restatable}[Convergence to the Equal-Weighted Average]{theorem}{RASGDIterationBound}
  \label{thm:r-asgd-bound}
    Let the worker-specific stepsizes be chosen according to \eqref{eq:learning-rates-equal} and suppose \Cref{assumption:computation-times,assumption:unbiasedness,assumption:bounded-gradient-variance,assumption:global-function-nc,assumption:local-f-smooth,assumption:bounded-heterogeneity} hold.
    If the stepsize parameter satisfies $0 < \gamma \le \min\nbr{\frac{1}{6 L \tau_H}, \frac{1}{5 L_{\max} \rho \tau_{\max}}}$,
    then, after $M \ge 1$ cycles, the cycle iterates $\xx^m$ generated by \Cref{algo:asgd-iv} satisfy
    \begin{equation}
      \label{eq:rasgd-norm-bound}
        \frac{1}{M} \sum_{m=0}^{M-1} \Exp{\norm{\nabla F(\xx^m)}^2} \le c_0 \cdot \frac{\Delta}{\gamma \tau_H M} + c_1 \cdot \frac{\gamma \tau_A L \sigma^2}{K} + c_2 \cdot \gamma^2 \tau_A^2 L_{\max}^2 \rbr{\sigma^2 + \zeta^2} 
    \end{equation}
    for some absolute constants $c_0, c_1, c_2 > 0$.
\end{restatable}

Note that the bound on the expected gradient norm is established for the cycle iterates $\xx^m$ only.
Over the course of a cycle, the inner iterates are pulled toward the workers' local objectives, whose stationary points do not generally coincide with those of the global objective.
The presence of $K$, the total number of updates per cycle, in the second term reflects this cycle-level perspective.

The last term in \eqref{eq:rasgd-norm-bound} reflects the effect of the cycle bias and scales with $\gamma^2$.
By scaling down the stepsize parameter $\gamma$, we can shrink this term and achieve $\varepsilon$-stationarity for arbitrarily small $\varepsilon > 0$.
The following corollary provides a bound on the worst-case time complexity to reach such an $\varepsilon$-stationary point.

\begin{restatable}[Time Complexity for the Equal-Weighted Average]{corollary}{RASGDTimeComplexity}
  \label{corollary:r-asgd-time-complexity}
    In the setting of \Cref{thm:r-asgd-bound}, the worst-case wall-clock time complexity to find an $\varepsilon$-stationary point of $F$ is 
    \begin{equation}
      \mathcal{O}\rbr{\frac{\Delta L \sigma^2}{n \varepsilon^2} \tau_A + \frac{\Delta L_{\max} \sqrt{\sigma^2 + \zeta^2}}{\varepsilon^{1.5}} \frac{\tau_{\max}}{\tau_H} \tau_A + \frac{\Delta L_{\max} \rho}{\varepsilon} \frac{\tau_{\max}}{\tau_H} \tau_{\max} + \frac{\Delta L}{\varepsilon} \tau_{\max}} .
    \end{equation}
\end{restatable}

This time complexity is notable for several reasons. 
First, its first and last terms, $\frac{\Delta L \sigma^2}{n \varepsilon^2} \tau_A$ and $\frac{\Delta L}{\varepsilon} \tau_{\max}$, match the theoretical lower bound established by \citet{tyurin_optimal_2023}.
The middle terms reflect penalties incurred from the combination of data heterogeneity and staleness.

For small $\varepsilon$, the leading term dominates the complexity bound.
This scales efficiently with the arithmetic mean of computation times, $\tau_A$, rather than their maximum, $\tau_{\max}$.
The local function parameters ($L_{\max}$) and data heterogeneity measures ($\zeta, \rho$) affect only lower-order terms, showing that the slowdown caused by function heterogeneity and gradient staleness is not of first-order concern.
The $\nicefrac{\tau_{\max}}{\tau_H}$ multiplier in these terms represents the average number of gradients delivered per worker during a cycle and captures the impact of staleness.


\paragraph{Objective Inconsistency under Equal Stepsizes} 
\label{par:vanilla_sgd_targets_a_rate_weighted_average}

Our analysis allows us to precisely quantify the objective inconsistency for other stepsize choices (see \Cref{thm:convergence-bound-general} for a more general result).
In \vasgd, all workers use the same stepsize $\gamma_i = \nicefrac{\gamma}{K}$, so that $\nicefrac{\gamma_i}{\tau_i} \propto \nicefrac{1}{\tau_i}$.
Consequently, \Cref{algo:asgd-iv} now targets a frequency-weighted average of the local functions,
\begin{equation}
  \label{eq:f-tilde}
  \tilde{F}(\xx) \coloneqq \sum_{i=1}^n \tilde{w}_i F_i(\xx) , \quad \tilde{w}_i \coloneqq \frac{\tau_i^{-1}}{\sum_{j=1}^n \tau_j^{-1}} \propto \tau_i^{-1} .
\end{equation}
Our next theorem makes this precise.

\begin{restatable}[Convergence to the Frequency-Weighted Average]{theorem}{VASGDIterationBound}
  \label{thm:v-asgd-bound}
    Let the workers' stepsizes be equal, $\gamma_i = \nicefrac{\gamma}{K}$, and 
    suppose \Cref{assumption:computation-times,assumption:unbiasedness,assumption:bounded-gradient-variance,assumption:global-function-nc,assumption:local-f-smooth,assumption:bounded-heterogeneity} hold for $\tilde{F}$.\footnote{That is, we replace $F, \Delta, L, \zeta, \rho$ by $\tilde{F}, \tilde\Delta, \tilde{L}, \tilde\zeta, \tilde\rho$ in \Cref{assumption:global-function-nc,assumption:bounded-heterogeneity}.}
    If the stepsize parameter satisfies $0 < \gamma \le \min\nbr{\frac{1}{6 \tilde{L}}, \frac{1}{5 L_{\max} \tilde{\rho}}}$,
    then, after $M \ge 1$ cycles, the cycle iterates $\xx^m$ generated by \Cref{algo:asgd-iv} satisfy
    \begin{equation*}
        \frac{1}{M} \sum_{m=0}^{M-1} \Exp{\norm{\nabla \tilde{F}(\xx^m)}^2} \le c_0 \cdot \frac{\tilde\Delta}{\gamma M} + c_1 \cdot \frac{\gamma \tilde{L} \sigma^2}{K} + c_2 \cdot \gamma^2 L_{\max}^2 \rbr{\sigma^2 + \tilde{\zeta}^2} 
    \end{equation*}
    for some absolute constants $c_0, c_1, c_2 > 0$.
\end{restatable}

As in \Cref{thm:r-asgd-bound}, the bias term on the right-hand side scales with $\gamma^2$, so that \Cref{algo:asgd-iv} converges to an $\varepsilon$-stationary point of the frequency-weighted average $\tilde{F}$ if the stepsize is sufficiently small.
To better understand the workings of \Cref{algo:asgd-iv} with equal stepsizes, we must consider its wall-clock time complexity.

\begin{restatable}[Time Complexity for the Frequency-Weighted Average]{corollary}{VASGDTimeComplexity}
  \label{corollary:v-asgd-time-complexity}
    In the setting of \Cref{thm:v-asgd-bound}, the worst-case wall-clock time complexity to find an $\varepsilon$-stationary point of $\tilde{F}$ is 
    \begin{equation}
      \mathcal{O}\rbr{\frac{\tilde\Delta \tilde{L} \sigma^2}{n \varepsilon^2} \tau_H + \frac{\tilde\Delta L_{\max} \sqrt{\sigma^2 + \tilde{\zeta}^2}}{\varepsilon^{1.5}} \tau_{\max} + \frac{\tilde\Delta L_{\max} \tilde{\rho}}{\varepsilon} \tau_{\max} + \frac{\tilde\Delta \tilde{L}}{\varepsilon} \tau_{\max}} .
    \end{equation}
\end{restatable}

Note that the leading term of this complexity bound now scales with the harmonic mean $\tau_H$.
This is in contrast to \Cref{corollary:r-asgd-time-complexity}, where the arithmetic mean $\tau_A$ appears instead. 
This captures the price of neutralizing objective inconsistency: 
In \vasgd, fast workers take unscaled steps, driving the global model forward at a higher rate governed by the harmonic mean of the workers' computation speeds. 
With rescaled stepsizes, on the other hand, stepsizes of fast workers are shrunk ($\gamma_i \propto \tau_i$) to prevent them from dominating the optimization trajectory. 
While this ensures convergence to the equal-weighted average, scaling down the updates from faster workers slows the global learning progress, shifting the time complexity bottleneck from the smaller harmonic mean $\tau_H$ to the larger arithmetic mean $\tau_A$.



\section{Experiments} 
\label{sec:experiments}

We compare \rasgd against \malenia and \ringleader, two state-of-the-art methods with optimal or near-optimal theoretical wall-clock time complexities, in a data-heterogeneous setup.
We train a two-layer neural network on MNIST \citep{lecun_mnist_2010}.
To enforce maximal heterogeneity, we partition the data by label \citep{hsu_measuring_2019} across $n=10$ workers, such that each worker holds images from exactly one class.
See \Cref{sec:experimental_details} for setup and methodology details.

\begin{figure}[htbp]
  \centering
  \begin{subfigure}[b]{0.48\textwidth}
    \centering
    \includegraphics[width=\textwidth]{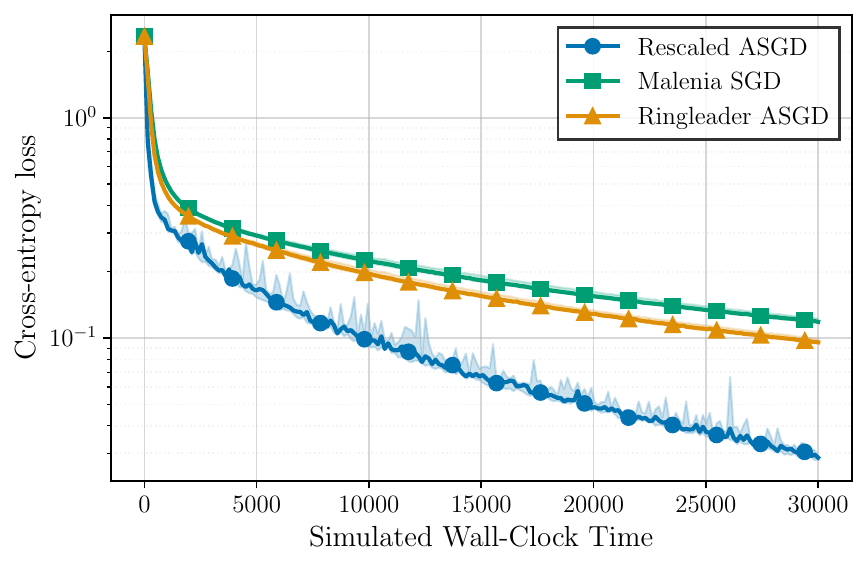}
    \caption{Fixed Times.}
    \label{fig:loss_harmonic}
  \end{subfigure}
  \hfill
  \begin{subfigure}[b]{0.48\textwidth}
    \centering
    \includegraphics[width=\textwidth]{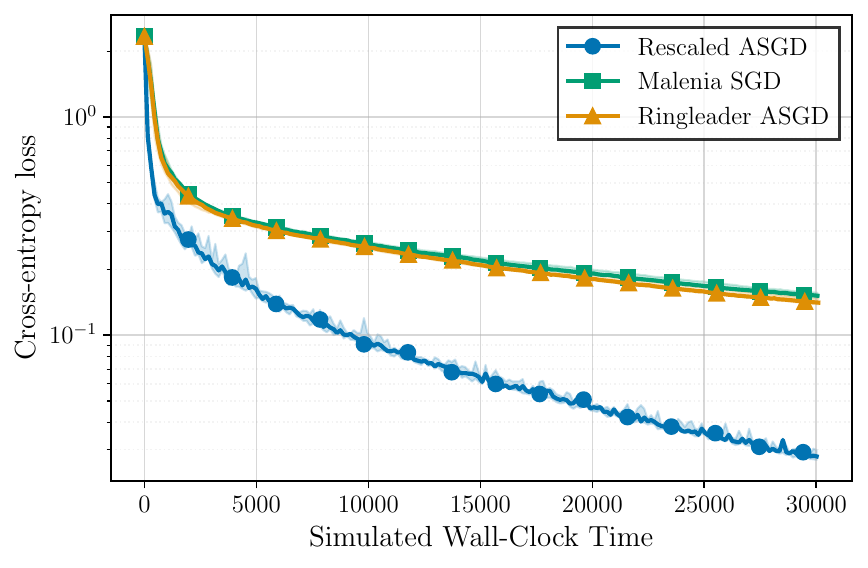}
    \caption{Fluctuating Times.}
    \label{fig:loss_exponential}
  \end{subfigure}
  \caption{Solid lines denote the median loss across five random seeds, with shaded regions indicating the minimum and maximum. \malenia and \ringleader are slowed down under fluctuating computation times as the gradient gathering phase takes longer. \rasgd shows virtually identical performance in both settings.}
  \label{fig:loss}
\end{figure}

We evaluate two computation time settings. 
First, fixed harmonic periods (\Cref{assumption:computation-times}) matching our theoretical setup.
We assign $\tau_i \in \{1,2,4,8,16\}$ to two workers each, making the fastest workers 16 times faster than the slowest.
Second, a fluctuating setting where computation times of worker $i$ are sampled from an exponential distribution with mean $\tau_i$.
While outside our theoretical framework, the underlying mechanism of equalized long-run learning progress remains.

\Cref{fig:loss} shows the loss trajectory. 
Under fluctuating times, \malenia and \ringleader degrade slightly due to stragglers slowing the gathering phase, whereas \rasgd remains unaffected.

To illustrate algorithmic differences, \Cref{fig:dynamics} plots the cumulative stepsizes over time.
For a fair comparison, worker $i$ takes a step of $\nicefrac{\alpha \tau_i}{n \tau_{\max}}$ in \rasgd, ensuring the total stepsize per cycle, taking $\tau_{\max}$ units of time, is $\alpha$.
\ringleader workers take steps of $\nicefrac{\alpha}{n}$ to match this sum, while \malenia requires no rescaling as it takes one step per cycle.

In the fixed setup (a), all methods share an average progress rate of $\alpha$ per $\tau_{\max}$ units of time, with \rasgd and \ringleader yielding smoother trajectories.
Under fluctuating times (b), \malenia and \ringleader are bottlenecked by the slowest worker due to their gathering phases.
Although faster workers compute continuously, the server model stalls until the straggler finishes, making the expected gathering duration greater than $\tau_{\max}$.
\rasgd continuously applies updates, exploring the search space unimpeded.

\begin{figure}[htbp]
  \centering
  \begin{subfigure}[b]{0.48\textwidth}
    \centering
    \includegraphics[width=\textwidth]{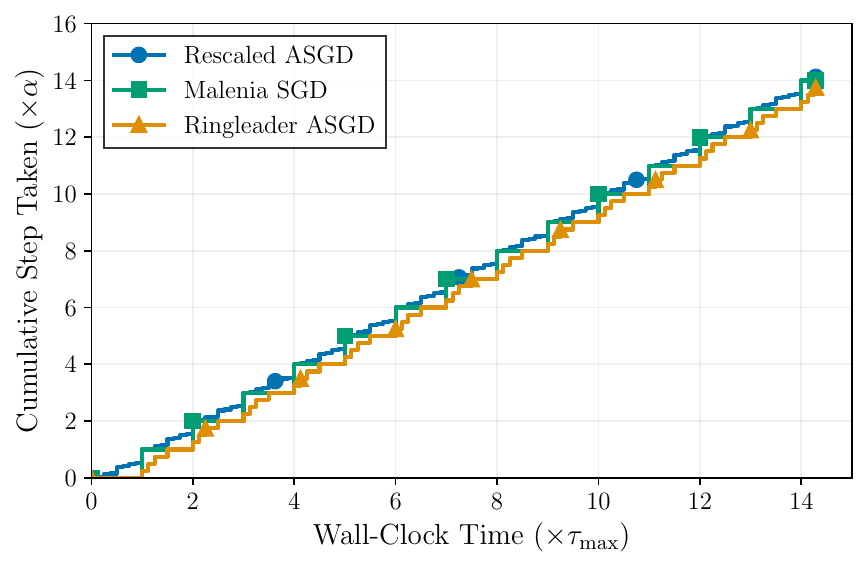}
    \caption{Fixed.}
    \label{fig:dynamics_a}
  \end{subfigure}
  \hfill
  \begin{subfigure}[b]{0.48\textwidth}
    \centering
    \includegraphics[width=\textwidth]{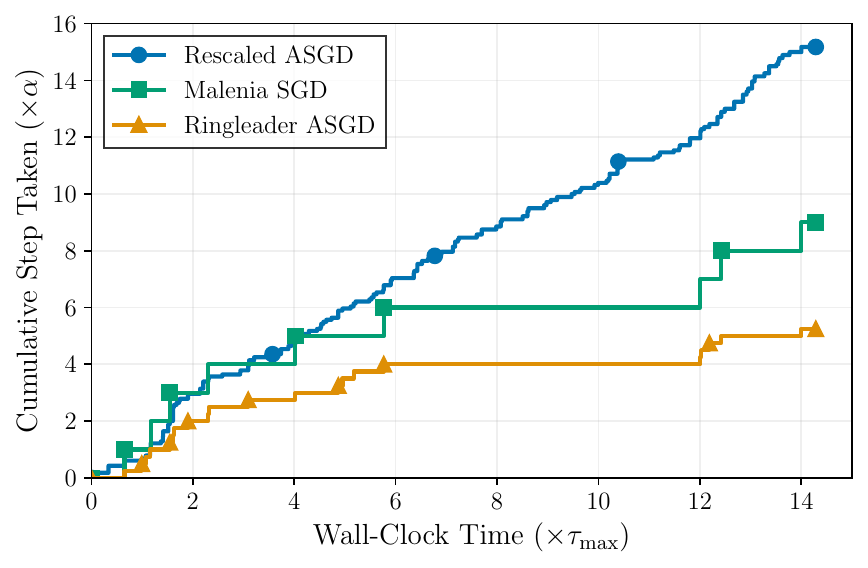}
    \caption{Fluctuating.}
    \label{fig:dynamics_b}
  \end{subfigure}
  \caption{Cumulative stepsize taken over wall-clock time. The gathering phases dilate under fluctuating computation times, causing \malenia and \ringleader to stall in the search space. \rasgd takes steps at an unimpeded rate of $\nicefrac{\alpha}{\tau_{\max}}$.}
  \label{fig:dynamics}
\end{figure}


\section{Conclusion} 
\label{sec:conclusion}

In this work, we introduced \rasgd to address the problem of objective inconsistency in asynchronous optimization caused by the interplay of data and system heterogeneity. 
By proportionally rescaling worker-specific stepsizes, we proved that \asyncsgd converges to the true global objective without introducing memory overhead, gathering phases, or worker idle times. 
Furthermore, our theoretical wall-clock time complexity matches known lower bounds in the leading term for the fixed-computation model considered here.

A primary limitation of our theoretical analysis is the reliance on the fixed-computation model with harmonic periods (\Cref{assumption:computation-times}). 
While our empirical results demonstrate that the method remains robust under stochastically fluctuating computation times, formally relaxing this structural assumption remains an open problem, likely requiring a different analytical approach. 
Future work may also explore extending this stepsize-rescaling principle to time-varying system heterogeneity.


\subsubsection*{Acknowledgments}
The research reported in this publication was supported by funding from King Abdullah University of Science and Technology (KAUST): i) KAUST Baseline Research Scheme, ii) CRG Grant ORFS-CRG12-2024-6460, and iii) Center of Excellence for Generative AI, under award number 5940.


{
\small
\bibliographystyle{plainnat}
\bibliography{references}
}


\newpage
\tableofcontents

\newpage
\appendix

\section{Related Work} 
\label{sec:related_literature}

\paragraph{Asynchronous Optimization} 
\label{par:asynchronous_sgd}

Asynchronous methods have been a cornerstone of parallel computing for over half a century \citep{chazan1969chaotic,baudet1978asynchronous,bertsekas2015parallel}.
Early research focused primarily on the homogeneous data setting, where the central challenge is managing stale gradients to ensure they do not derail the learning trajectory.
In these settings, the assumption of bounded delays has proven a fruitful abstraction for analysis \citep{agarwal_distributed_2011,recht_hogwild_2011}, a property that is naturally satisfied within the structured computation model we employ.
This line of work includes asynchronous nonconvex \sgd, perturbation-based analyses, tight delay-dependent bounds, and asynchronous variance-reduced methods \citep{lian_asynchronous_2015,mania_perturbed_2017,arjevani_tight_2020,j_reddi_variance_2015,leblond_asaga_2017,leblond_improved_2018,feyzmahdavian_asynchronous_2023,islamov_asgrad_2024}.
The more complex interplay between asynchronous updates and heterogeneous data has received comparatively less attention.

\citet{mishchenko_asynchronous_2022} examine a framework with arbitrary, unbounded delays. 
While their results for homogeneous data yield convergence rates to $\varepsilon$-stationary points comparable to ours, the heterogeneous case presents a fundamental challenge. 
Specifically, if worker computation times can vary without structure, the optimization trajectory becomes biased toward faster workers. Consequently, their guarantees in this setting necessitate a strong similarity assumption between local objectives and do not allow for convergence to an arbitrary $\varepsilon$-stationary point. 
By contrast, the structured nature of our computation model allows us to neutralize this bias through stepsize rescaling, enabling stronger guarantees without requiring local objectives to be nearly identical.

\citet{koloskova_sharper_2022} also explore asynchronous SGD under heterogeneous data but utilize a distinct mechanism that deviates from the standard \asyncsgd formulation (\Cref{algo:asgd}). 
In their model, the server assigns gradient computation tasks to workers via uniform random sampling. 
This process allows for multiple gradient tasks to be queued for a single worker. 
While this ensures that all workers contribute an equal number of gradients in the long run--matching the core motivation of our approach--it does so without utilizing faster workers' speed to increase throughput. 
In their model, faster workers are characterized by lower average delays rather than higher update frequencies. 
While this enables convergence to arbitrary $\varepsilon$-stationarity, the queuing mechanism does not reflect the common physical reality of asynchronous systems where workers deliver results as soon as they are computed and the server applies them as soon as they arrive. 
Our approach, \rasgd, allows workers to operate at their maximum physical frequency while using rescaled stepsizes to ensure that the resulting trajectory targets the correct global objective.


\paragraph{Objective Inconsistency} 
\label{par:objective_inconsistency}

The phenomenon wherein distributed methods fail to minimize the global objective \eqref{eq:objective-intro} when workers have different computation times is a well-documented challenge in distributed optimization \citep{wang_tackling_2020,islamov_asgrad_2024}. 
This inconsistency arises because the stochastic process underlying the optimization becomes implicitly weighted by the relative frequencies of worker updates, effectively targeting a surrogate objective $\tilde{F}$ rather than the equal-weighted average $F$.

\citet{wang_tackling_2020} address this challenge within the context of \textit{local SGD}, a framework where workers perform multiple local updates before periodically synchronizing with a central server \citep{mangasarian1995parallel,stich_local_2019,gorbunov_local_2020,woodworth_is_2020}. 
In their analysis, the number of local steps taken by a worker is intrinsically linked to its computation speed. 
To neutralize the resulting bias, they propose a mechanism conceptually similar to ours: 
rescaling the aggregate model difference from each worker by the reciprocal of the number of local steps performed. 
This normalization ensures that fast workers do not pull the global model disproportionately toward their local optima, thereby enabling the algorithm to reach arbitrary $\varepsilon$-stationarity for the true objective. 

Other strategies designed to mitigate bias from heterogeneous worker participation include the use of proximal regularization to constrain local drift \citep{li_federated_2020-1}, control variates to correct for client-server residuals \citep{karimireddy_scaffold_2021, mishchenko_proxskip_2022}, and staleness-aware mixing coefficients that de-weight delayed updates \citep{xie_asynchronous_2019}. 
While effective at improving stability, these methods typically focus on reducing the variance or the impact of stale information rather than enforcing the structural update parity required to eliminate objective inconsistency. 
In the absence of explicit rescaling or contribution equalization, the underlying optimization target remains skewed toward more active participants. 

Our work extends the rescaling principle to the asynchronous, single-update regime, providing a memory-efficient solution.


\paragraph{Optimal Methods in the Fixed-Computation Model} 
\label{par:optimal_methods}

The fixed-computation model offers a benchmark for evaluating the efficiency of parallel optimization algorithms. 
\citet{tyurin_optimal_2023} derive a lower bound for the wall-clock time complexity of first-order stochastic gradient methods when worker computation times $\tau_i$ are constant yet heterogeneous. 
While they propose synchronous methods that achieve this bound in both homogeneous and heterogeneous settings, their approach inherently sacrifices the continuous throughput of asynchronous systems to maintain synchronization.
Notably, in the homogeneous data setting, \vasgd can match this lower bound in regimes with high stochastic gradient variance $\sigma^2$ relative to the target precision $\varepsilon$.
Subsequent work extended the corresponding lower-bound framework beyond fixed worker computation times to arbitrary computation dynamics \citep{tyurin_tight_2024} and to variance-reduced parallel stochastic methods \citep{tovmasyan2026rennala_mvr}.
Complementary to these lower-bound results, \citet{maranjyan_first_2025} developed asynchronous methods that match the lower bounds on time complexities.

\citet{maranjyan_ringmaster_2025} improve upon these results for the homogeneous setting by introducing a variant of \asyncsgd that employs delay-adaptive learning rates and a hard-thresholding rule. 
By discarding gradients with staleness exceeding a specific threshold, they demonstrate that staleness bias can be actively managed to achieve optimal convergence regardless of the gradient variance. 
However, such a thresholding mechanism is fundamentally incompatible with the heterogeneous data setting, where it could systematically exclude updates from the slowest workers, thereby introducing the very objective inconsistency our work seeks to resolve.

\citet{maranjyan_ringleader_2026} propose an asynchronous algorithm, \ringleader, designed specifically for the heterogeneous regime. 
While their method approaches the optimal time complexity, its leading term is still constrained by a constant related to local objective similarity, similar to the limitations observed in the work of \citet{koloskova_sharper_2022} (cf. \Cref{table:complexities}). 
In contrast, \rasgd targets the equal-weighted average through stepsize rescaling rather than delay manipulation, gradient tables, or worker selection. 
This allows us to exploit the higher update frequency of faster workers, a benefit also noted in the empirical performance of \citet{maranjyan_ringleader_2026}. 
By ensuring that every worker's contribution is appropriately weighted in the model space, we maintain the exploration advantages of frequent asynchronous updates.



\section{Proofs} 
\label{sec:proofs}

\subsection{Preliminaries} 
\label{sub:preliminaries}

\paragraph{Notation} 
\label{par:notation}

In our model, randomness enters through the stochastic gradients $\nabla f_{i_k}(\yy_{k,i_k}^m, \xi_{k,i_k}^m)$ alone.
For ease of notation, we drop the dependence on the samples $\xi_{k,i_k}^m \sim \cD_{i_k}$ and write $\nabla f_{i_k}(\yy_{k,i_k}^m)$ instead.
We denote by $\cF^m$ the sigma field generated by the stochastic gradients delivered up to the beginning of cycle $m$.
Likewise, $\cF_k^m$ denotes the sigma field generated by the stochastic gradients delivered up to the beginning of iteration $k$ in cycle $m$, and define $\cF_K^m \coloneqq \cF^{m+1}$.
The iterates $\xx_k^m$ are $\cF_k^m$-measurable, as are $\yy_k^m$, while the stochastic gradients $\nabla f_{i_k}(\yy_{k,i_k}^m)$ are $\cF_{k+1}^m$-measurable as these are resolved only at the end of iteration $k$ in cycle $m$.

$\Exp{\bar{\xx}}$ denotes the unconditional expectation of a random vector $\bar{\xx}$, $\ExpSub{m}{\bar{\xx}}$ the conditional expectation conditional on all information available at the beginning of cycle $m$, i.e., conditioned on the sigma field $\cF^m$, $\ExpSub{m,k}{\bar{\xx}}$ the conditional expectation conditional on all information available at the beginning of iteration $k$ within cycle $m$, i.e., conditioned on the sigma field $\cF_k^m$.

Similarly, 
$$
  \var[_m]{\bar{\xx}} \coloneqq \mathbb{E}_m[\norm{\bar\xx - \ExpSub{m}{\bar{\xx}}}^2] = \mathbb{E}_m[\norm{\bar{\xx}}^2] - \norm{\mathbb{E}_m[\bar{\xx}]}^2
$$
denotes the conditional variance of a random vector $\bar{\xx}$ conditioned on all information available at the beginning of cycle $m$,
and
$$
  \cov[_m]{\bar{\xx},\bar{\yy}} \coloneqq \ExpSub{m}{\abr{\bar{\xx} - \ExpSub{m}{\bar{\xx}} , \bar{\yy} - \ExpSub{m}{\bar{\yy}}}} = \ExpSub{m}{\abr{\bar{\xx},\bar{\yy}}} - \abr{\ExpSub{m}{\bar{\xx}}, \ExpSub{m}{\bar{\yy}}}
$$
denotes the conditional covariance of two random vectors $\bar{\xx}, \bar{\yy}$.

We define
$$
  L_{\max} = \max_{i=1,\dots,n} L_i,\quad \gamma_{\max} \coloneqq \max_{i=1,\dots,n} \gamma_i,\quad \tau_{\max} \coloneqq \max_{i=1,\dots,n} \tau_i ~.
$$
Moreover,
$$
  \tau_A \coloneqq \frac{1}{n} \sum_{i=1}^n \tau_i \quad\text{and}\quad \tau_H \coloneqq \frac{n}{\sum_{i=1}^n \tau_i^{-1}}
$$
denote the arithmetic and harmonic mean of computation times, respectively.


\paragraph{Frequently used Assumptions, Definitions, and Standard Results} 
\label{par:frequently_used_assumptions_definitions_and_standard_results}

Below, we state a series of standard results, definitions, and the assumptions laid out in \Cref{sec:problem_setup} in simplified form for reference.
These will be used extensively in the following proofs.

Unbiased Gradients:
\begin{equation}\textstyleifpreprint
  \label{ass:unbiased}
  \ExpSub{\xi \sim D_i}{\nabla f_i(\xx, \xi)} = \nabla F_i(\xx) \tag{\text{UG}}
\end{equation}

Bounded Gradient Variance:
\begin{equation}\textstyleifpreprint
  \label{ass:variance}
  \ExpSub{\xi \sim D_i}{\norm{\nabla f_i(\xx, \xi) - \nabla F_i(\xx)}^2} \le \sigma^2 \tag{\text{BV}}
\end{equation} 

Smoothness:
\begin{equation}\textstyleifpreprint
  \label{ass:smooth}
  \norm{\nabla F_i(\xx) - \nabla F_i(\yy)} \le L_i \norm{\xx - \yy} , \norm{\nabla F(\xx) - \nabla F(\yy)} \le L \norm{\xx - \yy} \quad  \tag{\text{LS}}
\end{equation}

Lower Bound:
\begin{equation}\textstyleifpreprint
  \label{ass:lower-bound}
  F(\xx^0) - F(\xx^M) \ge F(\xx^0) - F^* = \Delta \tag{\text{LB}}
\end{equation}

Bounded Data Heterogeneity:
\begin{equation}\textstyleifpreprint
  \label{ass:bounded-heterogeneity}
  \norm{\nabla F_i(\xx)}^2 \le \zeta^2 + \rho^2 \norm{\nabla F(\xx)}^2 \tag{\text{BH}}
\end{equation}

Cycle Stepsize:
\begin{equation}\textstyleifpreprint
  \label{eq:alpha-def}
  \alpha \coloneqq \sum_{k=0}^{K-1} \gamma_{i_k} \tag{$\alpha$}
\end{equation}

Sum of Squared Stepsizes:
\begin{equation}\textstyleifpreprint
  \label{eq:A-def}
  A \coloneqq \sum_{k=0}^{K-1} \gamma_{i_k}^2 \tag{\text{A}}
\end{equation}

Cycle Noise:
\begin{equation}\textstyleifpreprint
  \label{eq:noise-def}
  \boldsymbol{\nu}_m \coloneqq \sum_{k=0}^{K-1} \gamma_{i_k} \rbr{\nabla f_{i_k}(\yy_{k,i_k}^m) - \nabla F_{i_k}(\yy_{k,i_k}^m)} \tag{\text{N}}
\end{equation}

Cycle Bias:
\begin{equation}\textstyleifpreprint
  \label{eq:bias-def}
  \mathbf{b}_m \coloneqq \sum_{k=0}^{K-1} \gamma_{i_k} \rbr{\nabla F_{i_k}(\yy_{k,i_k}^m) - \nabla F_{i_k}(\xx^m)} \tag{\text{B}}
\end{equation}

Triangle Inequality:
\begin{equation}\textstyleifpreprint
  \label{eq:triangle}
  \norm{\sum_{i=1}^n \vv_i} \le \sum_{i=1}^n \norm{\vv_i} \tag{\text{T}}
\end{equation}

Cauchy-Schwarz Inequality:
\begin{equation}\textstyleifpreprint
  \label{eq:cauchy-schwarz}
  \rbr{\sum_{i=1}^d u_i v_i}^2 \le \rbr{\sum_{i=1}^d u_i^2} \rbr{\sum_{i=1}^d v_i^2} \tag{\text{CS}}
\end{equation}

Tower Property:
\begin{equation}\textstyleifpreprint
  \label{eq:tower-property}
  \ExpSub{m}{\ExpSub{m,k}{\bar\xx}} = \ExpSub{m}{\bar\xx} \tag{\text{TP}}
\end{equation}

Descent Lemma:
\begin{equation}\textstyleifpreprint
  \label{eq:descent-lemma}
  F(\yy) \le F(\xx) + \abr{\nabla F(\xx), \yy - \xx} + \frac{L}{2} \norm{\yy - \xx}^2 \tag{\text{DL}}
\end{equation}

Jensen's Inequality:
\begin{equation}\textstyleifpreprint
  \label{eq:jensen}
  \norm{\ExpSub{m}{\bar\xx}}^2 \le \ExpSub{m}{\norm{\bar\xx}^2} \tag{\text{JN}}
\end{equation}

Young's Inequality:
\begin{equation}\textstyleifpreprint
  \label{eq:young-fancy}
  \norm{\uu} \norm{\vv} \le \frac{s}{2} \norm{\uu}^2 + \frac{1}{2s} \norm{\vv}^2 , \quad \forall s > 0 \tag{\text{YN}}
\end{equation}

Squared Sum Inequality:
\begin{equation}\textstyleifpreprint
  \label{eq:squared-sum}
  \norm{\sum_{i=1}^n \vv_i}^2 \le n \sum_{i=1}^n \norm{\vv_i}^2 \tag{\text{SS}}.
\end{equation}


\paragraph{Proof Outline} 
\label{par:proof_outline}

\Cref{thm:convergence-bound-general} states our main result in a more general form, showing that \rasgd can target any convex combination \eqref{eq:objective-appendix} of the local objective functions, and establishes a bound on the expected average squared gradient norm of the cycle iterates.
To that end, \Cref{thm:near-stationarity-abstract-bias} establishes a bound in terms of the squared norm of the cycle bias utilizing a standard descent lemma.
A bound on the norm of the expected cumulative cycle bias is derived in \Cref{lemma:bias-i,lemma:bias-ii,lemma:bias-iii,lemma:bias-iv,lemma:bias-term}.
Similarly, \Cref{lemma:noise-term} establishes a bound on the squared norm of the cycle noise.

\Cref{thm:r-asgd-bound,thm:v-asgd-bound} restate \Cref{thm:convergence-bound-general} with the appropriately chosen stepsizes.
\Cref{corollary:r-asgd-time-complexity,corollary:v-asgd-time-complexity} then follow immediately.



\subsection{Auxiliary Results} 
\label{sub:proof_of_thm:convergence-bound-general}

Throughout this subsection, we denote by
\begin{equation}
  \label{eq:objective-appendix}
  F(x) \coloneqq \sum_{i=1}^n w_i F_i(\xx)
\end{equation}
an arbitrary convex combination of the local objective functions $F_i$.
The weights $w_i \ge 0$ sum to unity and are understood to be fixed constants.
\Cref{assumption:global-function-nc,assumption:bounded-heterogeneity} are understood to refer to the objective function defined in \eqref{eq:objective-appendix}.

Note that we recover the equal-weighted average \eqref{eq:objective-intro} by setting $w_i = \nicefrac{1}{n}$.
The more general case presented in this subsection allows us to derive the results from the main text (\Cref{thm:r-asgd-bound,thm:v-asgd-bound}) as special cases of the more general \Cref{thm:convergence-bound-general}.

\begin{lemma}[Cycle Noise]
\label{lemma:noise-term}
    Under \cref{assumption:unbiasedness,assumption:bounded-gradient-variance}, the cycle noise 
    \begin{equation*}
      \boldsymbol{\nu}_m \coloneqq \sum_{k=0}^{K-1} \gamma_{i_k} \rbr{\nabla f_{i_k}(\yy_{k,i_k}^m) - \nabla F_{i_k}(\yy_{k,i_k}^m)}
    \end{equation*}
    satisfies
    \begin{equation*}
        \ExpSub{m}{\boldsymbol{\nu}_m} = 0 , \qquad
        \ExpSub{m}{\norm{\boldsymbol{\nu}_m}^2} \le A \sigma^2 .
    \end{equation*}
\end{lemma}

\begin{proof}
  Recall that $\yy_{k,i_k}^m$ denotes the model held by worker $i_k$ at the beginning of iteration $k$ of cycle $m$.
  Thus,
  \begin{eqnarray*}
      \ExpSub{m}{\boldsymbol{\nu}_m} & \overset{\eqref{eq:noise-def}}{=} & \ExpSub{m}{\sum_{k=0}^{K-1} \gamma_{i_k} \rbr{\nabla f_{i_k}(\yy_{k,i_k}^m) - \nabla F_{i_k}(\yy_{k,i_k}^m)}} \\
      & \overset{\eqref{eq:tower-property}}{=} & \sum_{k=0}^{K-1} \gamma_{i_k} \ExpSub{m}{\ExpSub{m,k}{\nabla f_{i_k}(\yy_{k,i_k}^m) - \nabla F_{i_k}(\yy_{k,i_k}^m)}} \\
      & \overset{\eqref{ass:unbiased}}{=} & \sum_{k=0}^{K-1} \gamma_{i_k} \ExpSub{m}{0} \\
      & = & 0 ~.
  \end{eqnarray*}

  Further, 
  \begin{eqnarray*}
      \ExpSub{m}{\norm{\boldsymbol{\nu}_m}^2} & = & \var[_m]{\boldsymbol{\nu}_m} + \norm{\ExpSub{m}{\boldsymbol{\nu}_m}}^2 \\
      & = & \var[_m]{\sum_{k=0}^{K-1} \gamma_{i_k} \rbr{\nabla f_{i_k}(\yy_{k,i_k}^m) - \nabla F_{i_k}(\yy_{k,i_k}^m)}} \\
      & = & \sum_{k = 0}^{K-1} \sum_{l = 0}^{K-1} \gamma_{i_k} \gamma_{i_l} \cov[_m]{\nabla f_{i_k}(\yy_{k,i_k}^m) - \nabla F_{i_k}(\yy_{k,i_k}^m), \nabla f_{i_l}(\yy_{l,i_l}^m) - \nabla F_{i_l}(\yy_{l,i_l}^m)} .
  \end{eqnarray*}

  Let $\ee_k^m \coloneqq \nabla f_{i_k}(\yy_{k,i_k}^m) - \nabla F_{i_k}(\yy_{k,i_k}^m)$.
  Unbiasedness of the stochastic gradients (\Cref{assumption:unbiasedness}) ensures $\ExpSub{m}{\ee_k^m} = \ExpSub{m}{\ExpSub{m,k}{\ee_k^m}} = 0$.
  Now consider a cross-term with $k < l$:
  \begin{eqnarray*}
      \cov[_m]{\ee_k^m, \ee_l^m} & = & \ExpSub{m}{\abr{\ee_k^m, \ee_l^m}} - \abr{\ExpSub{m}{\ee_k^m}, \ExpSub{m}{\ee_l^m}} \\
      & \overset{\eqref{eq:tower-property}}{=} & \ExpSub{m}{\ExpSub{m,l}{\abr{\ee_k^m, \ee_l^m}}} \\
      & = & \ExpSub{m}{\abr{\ee_k^m , \ExpSub{m,l}{\ee_l^m}}} \\
      & = & 0 ~,
  \end{eqnarray*}

  where we used that $\ee_k^m$ is $\cF_l^m$-measurable, i.e., its randomness has been resolved by the beginning of iteration $l > k$.

  With \Cref{assumption:bounded-gradient-variance}, we now find
  \begin{eqnarray*}
      \ExpSub{m}{\norm{\boldsymbol{\nu}_m}^2} & = & \var[_m]{\boldsymbol{\nu}_m} \\
      & = & \sum_{k=0}^{K-1} \gamma_{i_k}^2 \cov[_m]{\ee_k^m, \ee_k^m} \\
      & = & \sum_{k=0}^{K-1} \gamma_{i_k}^2 \ExpSub{m}{\norm{\ee_k^m}^2} \\
      & \overset{\eqref{ass:variance}}{\le} & \sum_{k=0}^{K-1} \gamma_{i_k}^2 \sigma^2 \\
      & = & A \sigma^2 ~.
  \end{eqnarray*}
\end{proof}

\begin{lemma}
    \label{lemma:bias-i}
    Under \Cref{assumption:local-f-smooth}, the cycle bias 
    \begin{equation*}
      \mathbf{b}_m \coloneqq \sum_{k=0}^{K-1} \gamma_{i_k} \rbr{\nabla F_{i_k}(\yy_{k,i_k}^m) - \nabla F_{i_k}(\xx^m)}
    \end{equation*}
    satisfies
    \begin{equation}
        \label{eq:bias-i}
        \norm{\mathbf{b}_m}^2 \le A L_{\max}^2 \sum_{k=0}^{K-1} \norm{\yy_{k,i_k}^m - \xx^m}^2 .
    \end{equation}
\end{lemma}

\begin{proof}
  Using our definitions of the cycle bias $\mathbf{b}_m$ and cycle stepsize $A$, we find
  \begin{eqnarray*}
      \norm{\mathbf{b}_m}^2 & \overset{\eqref{eq:bias-def}}{=} & \norm{\sum_{k=0}^{K-1} \gamma_{i_k} \rbr{\nabla F_{i_k}(\yy_{k,i_k}^m) - \nabla F_{i_k}(\xx^m)}}^2 \\
      & \overset{\eqref{eq:triangle}}{\le} & \rbr{\sum_{k=0}^{K-1} \gamma_{i_k} \norm{\nabla F_{i_k}(\yy_{k,i_k}^m) - \nabla F_{i_k}(\xx^m)}}^2 \\
      & \overset{\eqref{eq:cauchy-schwarz}}{\le} & \rbr{\sum_{k=0}^{K-1} \gamma_{i_k}^2} \rbr{\sum_{k=0}^{K-1} \norm{\nabla F_{i_k}(\yy_{k,i_k}^m) - \nabla F_{i_k}(\xx^m)}^2} \\
      & \overset{\eqref{eq:A-def}}{=} & A \sum_{k=0}^{K-1} \norm{\nabla F_{i_k}(\yy_{k,i_k}^m) - \nabla F_{i_k}(\xx^m)}^2 \\
      & \overset{\eqref{ass:smooth}}{\le} & A \sum_{k=0}^{K-1} L_{i_k}^2 \norm{\yy_{k,i_k}^m - \xx^m}^2 \\
      & \le & A L_{\max}^2 \sum_{k=0}^{K-1} \norm{\yy_{k,i_k}^m - \xx^m}^2 ,
  \end{eqnarray*}

  where we used the smoothness of the local functions (\Cref{assumption:local-f-smooth}) in the penultimate line.
\end{proof}
\begin{lemma}
    \label{lemma:bias-ii}
    Let $k' \in \nbr{0, \dots, K-1}, m' \in \nbr{m-1,m}$ denote the iteration and cycle index corresponding to the last model seen by worker $i_k$ at the beginning of iteration $k$ in cycle $m$, i.e., $\yy_{k,i_k}^m = \xx_{k'}^{m'}$.
    If we denote the set of iterations between $\yy_{k,i_k}^m$ and $\xx^m$ by 
    \begin{equation*}
      \cI_k^m \coloneqq \begin{cases} \nbr{k', \dots, K-1} \times \nbr{m-1} , & m' = m - 1 \\ \nbr{0, \dots, k' - 1} \times \nbr{m}, & m' = m \end{cases} , 
    \end{equation*}
    where $\nbr{0, \dots, k' - 1}$ is understood as the empty set when $k'=0$,
    then
    \begin{equation*}
        \norm{\yy_{k,i_k}^m - \xx^m}^2 \le K \sum_{(l,c) \in \cI_k^m} \gamma_{i_l}^2 \norm{\nabla f_{i_l}(\yy_{l,i_l}^c)}^2 .
    \end{equation*}
\end{lemma}
\begin{proof}
  Note that $\abs{\cI_k^m} \le K$ as the cyclic update schedule ensures that each worker receives a model update after at most $K$ iterations.
  The vector $\xx_{k'}^{m'} - \xx^m$ is, up to sign, the sum of the updates indexed by $\cI_k^m$; the sign depends on whether $m'=m-1$ or $m'=m$ and disappears after taking the norm.
  The claim then follows from the squared-sum inequality:
  \begin{eqnarray*}
      \norm{\yy_{k,i_k}^m - \xx^m}^2 & = & \norm{\xx_{k'}^{m'} - \xx^m}^2 \\
      & = & \norm{\sum_{(l,c) \in \cI_k^m} \gamma_{i_l} \nabla f_{i_l}(\yy_{l,i_l}^c)}^2 \\
      & \overset{\eqref{eq:squared-sum}}{\le} & \abs{\cI_k^m} \sum_{(l,c) \in \cI_k^m} \gamma_{i_l}^2 \norm{\nabla f_{i_l}(\yy_{l,i_l}^c)}^2 \\
      & \le & K \sum_{(l,c) \in \cI_k^m} \gamma_{i_l}^2 \norm{\nabla f_{i_l}(\yy_{l,i_l}^c)}^2 .
  \end{eqnarray*}
\end{proof}

\begin{lemma}
    \label{lemma:bias-iii}
    Under \cref{assumption:unbiasedness,assumption:bounded-gradient-variance,assumption:global-function-nc,assumption:bounded-heterogeneity},
    \begin{equation}
        \label{eq:bias-iii}
        \Exp{\norm{\nabla f_{i_k}(\yy_{k,i_k}^m)}^2} \le \sigma^2 + \zeta^2 + 2 \rho^2 \Exp{\norm{\nabla F(\xx^m)}^2} + 2 \rho^2 L^2 \Exp{\norm{\yy_{k,i_k}^m - \xx^m}^2} .
    \end{equation}
\end{lemma}

\begin{proof}
    By \cref{assumption:unbiasedness,assumption:bounded-gradient-variance},
    \begin{eqnarray*}
        \ExpSub{m,k}{\norm{\nabla f_{i_k}(\yy_{k,i_k}^m)}^2} & \overset{\eqref{ass:unbiased}}{=} & \ExpSub{m,k}{\norm{\nabla f_{i_k}(\yy_{k,i_k}^m) - \nabla F_{i_k}(\yy_{k,i_k}^m)}^2} + \ExpSub{m,k}{\norm{\nabla F_{i_k}(\yy_{k,i_k}^m)}^2} \\
        & \overset{\eqref{ass:variance}}{\le} & \sigma^2 + \norm{\nabla F_{i_k}(\yy_{k,i_k}^m)}^2 .
    \end{eqnarray*}

    For the latter term, using \cref{assumption:global-function-nc,assumption:bounded-heterogeneity}, we derive the bound
    \begin{eqnarray*}
        \norm{\nabla F_{i_k}(\yy_{k,i_k}^m)}^2 & \overset{\eqref{ass:bounded-heterogeneity}}{\le} & \zeta^2 + \rho^2 \norm{\nabla F(\yy_{k,i_k}^m)}^2 \\
        & \overset{\eqref{eq:squared-sum}}{\le} & \zeta^2 + 2 \rho^2 \norm{\nabla F(\xx^m)}^2 + 2 \rho^2 \norm{\nabla F(\yy_{k,i_k}^m) - \nabla F(\xx^m)}^2 \\
        & \overset{\eqref{ass:smooth}}{\le} & \zeta^2 + 2 \rho^2 \norm{\nabla F(\xx^m)}^2 + 2 \rho^2 L^2 \norm{\yy_{k,i_k}^m - \xx^m}^2 .
    \end{eqnarray*}

    Taking the unconditional expectation then gives
    \begin{eqnarray*}
        \Exp{\norm{\nabla f_{i_k}(\yy_{k,i_k}^m)}^2} & \le & \sigma^2 + \zeta^2 + 2 \rho^2 \Exp{\norm{\nabla F(\xx^m)}^2} + 2 \rho^2 L^2 \Exp{\norm{\yy_{k,i_k}^m - \xx^m}^2} .
    \end{eqnarray*}
\end{proof}

\begin{lemma}
  \label{lemma:bias-iv}
  Let
  \begin{equation*}
      Q_M \coloneqq \sum_{m=0}^{M-1} \sum_{k=0}^{K-1} \gamma_{i_k}^2 \Exp{\norm{\nabla f_{i_k}(\yy_{k,i_k}^m)}^2} .
  \end{equation*}
  For $\gamma_{\max} \le \frac{1}{2 K L \rho}$, we have
  \begin{equation*}
      Q_M \le 2 A M \rbr{\sigma^2 + \zeta^2} + 4 A \rho^2 \sum_{m=0}^{M-1} \Exp{\norm{\nabla F(\xx^m)}^2} .
  \end{equation*}
\end{lemma}
\begin{proof}
  Multiplying the bound derived in \Cref{lemma:bias-ii} by $\gamma_{i_k}^2$ and summing over all cycles $m=0, \dots, M-1$ and iterations $k=0, \dots, K-1$, we find
  \begin{eqnarray*}
    \sum_{m=0}^{M-1} \sum_{k=0}^{K-1} \gamma_{i_k}^2 \norm{\yy_{k,i_k}^m - \xx^m}^2 & \le & K \sum_{m=0}^{M-1} \sum_{k=0}^{K-1} \gamma_{i_k}^2 \sum_{(l,c) \in \cI_k^m} \gamma_{i_l}^2 \norm{\nabla f_{i_l}(\yy_{l,i_l}^c)}^2 \\
    & = & K \sum_{m'=0}^{M-1} \sum_{k'=0}^{K-1} \gamma_{i_{k'}}^2 \norm{\nabla f_{i_{k'}}(\yy_{k',i_{k'}}^{m'})}^2 \times \rbr{\sum_{(l,c) : (k',m') \in \cI_l^c} \gamma_{i_l}^2} .
  \end{eqnarray*}
  In the first line, we sum over the paths $\cI_k^m$ connecting $\yy_{k,i_k}^m$ and $\xx^m$, which we bound by considering the norms of the gradients computed along these paths, $\nabla f_{i_l}(\yy_{l,i_l}^c)$.
  To get to the second line, we instead consider the gradient in an iteration $(k',m')$, $\nabla f_{i_{k'}}(\yy_{k',i_{k'}}^{m'})$, and sum up the contribution of the paths on which it appears.
  Since the update schedule is cyclic and the stepsizes depend only on the worker index, the sum of squared stepsizes over any block of $K$ consecutive updates is $A$.
  For a fixed update $(k',m')$, all destination updates $(l,c)$ whose paths contain $(k',m')$ lie in such a block of at most one full cycle.
  We can therefore bound this factor uniformly by
  \begin{equation*}
    \sum_{(l,c) : (k',m') \in \cI_l^c} \gamma_{i_l}^2 \le \sum_{k=0}^{K-1} \gamma_{i_k}^2 = A ~.
  \end{equation*}
  Hence,
  \begin{eqnarray}
    \label{eq:q-aux-bound}
    \sum_{m=0}^{M-1} \sum_{k=0}^{K-1} \gamma_{i_k}^2 \Exp{\norm{\yy_{k,i_k}^m - \xx^m}^2} & \le & A K \sum_{m=0}^{M-1} \sum_{k=0}^{K-1} \gamma_{i_k}^2 \Exp{\sqnorm{\nabla f_{i_k}(\yy_{k,i_k}^m)}} = A K Q_M .
  \end{eqnarray}
  Now, using \Cref{lemma:bias-iii}, we have
  \begin{eqnarray*}
    Q_M & = & \sum_{m=0}^{M-1} \sum_{k=0}^{K-1} \gamma_{i_k}^2 \Exp{\norm{\nabla f_{i_k}(\yy_{k,i_k}^m)}^2} \\
    & \overset{\eqref{eq:bias-iii}}{\le} & \sum_{m=0}^{M-1} \sum_{k=0}^{K-1} \gamma_{i_k}^2 \rbr{\sigma^2 + \zeta^2 + 2 \rho^2 \Exp{\norm{\nabla F(\xx^m)}^2} + 2 L^2 \rho^2 \Exp{\norm{\yy_{k,i_k}^m - \xx^m}^2}} \\
    & = & A M \rbr{\sigma^2 + \zeta^2} + 2 A \rho^2 \sum_{m=0}^{M-1} \Exp{\norm{\nabla F(\xx^m)}^2} + 2 L^2 \rho^2 \sum_{m=0}^{M-1} \sum_{k=0}^{K-1} \gamma_{i_k}^2 \Exp{\norm{\yy_{k,i_k}^m - \xx^m}^2} \\
    & \overset{\eqref{eq:q-aux-bound}}{\le} & A M \rbr{\sigma^2 + \zeta^2} + 2 A \rho^2 \sum_{m=0}^{M-1} \Exp{\norm{\nabla F(\xx^m)}^2} + 2 A K L^2 \rho^2 Q_M \\
    & \le & A M \rbr{\sigma^2 + \zeta^2} + 2 A \rho^2 \sum_{m=0}^{M-1} \Exp{\norm{\nabla F(\xx^m)}^2} + 2 \gamma_{\max}^2 K^2 L^2 \rho^2 Q_M ~,
  \end{eqnarray*}

  where we use $A = \sum_{k=0}^{K-1} \gamma_{i_k}^2 \le \gamma_{\max} \sum_{k=0}^{K-1} \gamma_{i_k} \le \gamma_{\max}^2 K$ to obtain the last inequality.

  Gathering $Q_M$-terms on the left side,
  \begin{eqnarray*}
    \rbr{1 - 2 \gamma_{\max}^2 K^2 L^2 \rho^2} Q_M & \le & A M \rbr{\sigma^2 + \zeta^2} + 2 A \rho^2 \sum_{m=0}^{M-1} \Exp{\norm{\nabla F(\xx^m)}^2} ,
  \end{eqnarray*}
  we see that, if $\gamma_{\max} \le \frac{1}{2 K L \rho}$, then
  \begin{eqnarray*}
    Q_M & \le & 2 A M \rbr{\sigma^2 + \zeta^2} + 4 A \rho^2 \sum_{m=0}^{M-1} \Exp{\norm{\nabla F(\xx^m)}^2} .
  \end{eqnarray*}
\end{proof}
\begin{lemma}[Bias]
\label{lemma:bias-term}
    Under \Cref{assumption:unbiasedness,assumption:bounded-gradient-variance,assumption:global-function-nc,assumption:local-f-smooth,assumption:bounded-heterogeneity}, and if $\gamma_{\max} \le \frac{1}{2 K L \rho}$, then
    \begin{equation}
        \sum_{m=0}^{M-1} \Exp{\norm{\mathbf{b}_m}^2} \le 2 A^2 K^2 L_{\max}^2 M \rbr{\sigma^2 + \zeta^2} + 4 A^2 K^2 L_{\max}^2 \rho^2 \sum_{m=0}^{M-1} \Exp{\norm{\nabla F(\xx^m)}^2} .
    \end{equation}
\end{lemma}

\begin{proof}
  Similar to the proof of \Cref{lemma:bias-iv}, we first bound
  \begin{eqnarray*}
    \sum_{m=0}^{M-1} \sum_{k=0}^{K-1} \norm{\yy_{k,i_k}^m - \xx^m}^2 & \le & K \sum_{m=0}^{M-1} \sum_{k=0}^{K-1} 1 \times \sum_{(l,c) \in \cI_k^m} \gamma_{i_l}^2 \norm{\nabla f_{i_l}(\yy_{l,i_l}^c)}^2 \\
    & = & K \sum_{m'=0}^{M-1} \sum_{k'=0}^{K-1} \gamma_{i_{k'}}^2 \norm{\nabla f_{i_{k'}}(\yy_{k',i_{k'}}^{m'})}^2 \times \underbrace{\rbr{\sum_{(l,c) : (k',m') \in \cI_l^c} 1}}_{\le K} ,
  \end{eqnarray*}
  where the last factor is bounded by $K$ because, for a fixed update $(k',m')$, all destination updates $(l,c)$ whose paths contain $(k',m')$ must lie within at most one full cycle after $(k',m')$.
  After $K$ updates, every worker has received a model update newer than $(k',m')$, so no later path can contain this fixed update.
  So 
  \begin{equation}
    \label{eq:bias-aux}
    \sum_{m=0}^{M-1} \sum_{k=0}^{K-1} \Exp{\norm{\yy_{k,i_k}^m - \xx^m}^2} \le K^2 Q_M .
  \end{equation}
  Now
  \begin{eqnarray*}
    \sum_{m=0}^{M-1} \Exp{\norm{\mathbf{b}_m}^2} & \overset{\eqref{eq:bias-i}}{\le} & A L_{\max}^2 \sum_{m=0}^{M-1} \sum_{k=0}^{K-1} \Exp{\norm{\yy_{k,i_k}^m - \xx^m}^2} \\
    & \overset{\eqref{eq:bias-aux}}{\le} & A K^2 L_{\max}^2 Q_M .
  \end{eqnarray*}

  If $\gamma_{\max} \le \frac{1}{2 K L \rho}$, we can apply \Cref{lemma:bias-iv} to obtain the desired bound:
  \begin{eqnarray*}
      \sum_{m=0}^{M-1} \Exp{\norm{\mathbf{b}_m}^2} & \le & 2 A^2 K^2 L_{\max}^2 M \rbr{\sigma^2 + \zeta^2} + 4 A^2 K^2 L_{\max}^2 \rho^2 \sum_{m=0}^{M-1} \Exp{\norm{\nabla F(\xx^m)}^2} .
  \end{eqnarray*}
\end{proof}
\begin{lemma}[Cycle Step Decomposition]
  \label{lemma:cycle-step-decomposition}
  Let the worker-specific stepsizes be chosen such that $\gamma_i \propto w_i \tau_i$.
  Then,
    \begin{equation}
      \label{eq:cycle-step-decomposition}
      \mathbf{S}_m \coloneqq \sum_{k=0}^{K-1} \gamma_{i_k} \nabla f_{i_k}(\yy_{k,i_k}^m) = \alpha \nabla F(\xx^m) + \mathbf{b}_m + \boldsymbol{\nu}_m ~,
    \end{equation}

    where $\alpha \coloneqq \sum_{k=0}^{K-1} \gamma_{i_k}$ is the cycle stepsize.
\end{lemma}

\begin{proof}
  Let $\gamma_i = c w_i \tau_i$ for some $c > 0$.
  Then
  \begin{eqnarray*}
    \alpha & = & \sum_{k=0}^{K-1} \gamma_{i_k} \\
    & = & \sum_{i=1}^n \gamma_i K_i \\
    & = & \sum_{i=1}^n c w_i \tau_i \cdot \frac{\tau_{\max}}{\tau_i} \\
    & = & c \tau_{\max} ~,
  \end{eqnarray*}

  and consequently
  \begin{eqnarray}
    \sum_{k=0}^{K-1} \gamma_{i_k} \nabla F_{i_k}(\xx^m) & = & \sum_{i=1}^n \gamma_i K_i \nabla F_i(\xx^m) \notag \\
    & = & \sum_{i=1}^n c w_i \tau_i \cdot \frac{\tau_{\max}}{\tau_i} \cdot \nabla F_i(\xx^m) \notag \\
    & = & c \tau_{\max} \sum_{i=1}^n w_i \nabla F_i(\xx^m) \notag \\
    & = & \alpha \nabla F(\xx^m) ~. \label{eq:ideal-step}
  \end{eqnarray}

  For the cycle step, we now have
  \begin{eqnarray*}
      \mathbf{S}_m & = & \sum_{k=0}^{K-1} \gamma_{i_k} \nabla f_{i_k}(\yy_{k,i_k}^m) \\
      & = & \sum_{k=0}^{K-1} \gamma_{i_k} \nabla F_{i_k}(\xx^m) \\
      && + \sum_{k=0}^{K-1} \gamma_{i_k} \rbr{\nabla F_{i_k}(\yy_{k,i_k}^m) - \nabla F_{i_k}(\xx^m)} \\
      && \quad + \sum_{k=0}^{K-1} \gamma_{i_k} \rbr{\nabla f_{i_k}(\yy_{k,i_k}^m) - \nabla F_{i_k}(\yy_{k,i_k}^m)} \\
      & \overset{\eqref{eq:ideal-step},\eqref{eq:bias-def},\eqref{eq:noise-def}}{=} & \alpha \nabla F(\xx^m) + \mathbf{b}_m + \boldsymbol{\nu}_m ~.
  \end{eqnarray*}
\end{proof}

\begin{lemma}
\label{thm:near-stationarity-abstract-bias}
    Let the target objective be $F(\xx) = \sum_{i=1}^n w_i F_i(\xx)$, and suppose \Cref{assumption:computation-times,assumption:unbiasedness,assumption:bounded-gradient-variance,assumption:global-function-nc} hold. 
    Assume the worker-specific stepsizes are chosen such that $\gamma_i \propto w_i \tau_i$. 
    If the cycle stepsize $\alpha \coloneqq \sum_{k=0}^{K-1} \gamma_{i_k}$ satisfies $0 < \alpha \le \frac{1}{6 L}$, 
    then, after $M \ge 1$ cycles, the cycle iterates $\xx^m$ generated by \Cref{algo:asgd-iv} satisfy
    \begin{equation*}
        \frac{1}{M} \sum_{m=0}^{M-1} \Exp{\norm{\nabla F(\xx^m)}^2} \le \frac{2 \Delta}{\alpha M} + \frac{3 A L \sigma^2}{\alpha} + \rbr{\frac{2}{\alpha^2} + \frac{3 L}{\alpha}} \frac{1}{M} \sum_{m=0}^{M-1} \Exp{\norm{\mathbf{b}_m}^2} .
    \end{equation*}
\end{lemma}

\begin{proof}
  In \Cref{lemma:cycle-step-decomposition}, we defined the cycle step $\mathbf{S}_m$, so that for the cycle iterates of \Cref{algo:asgd-iv}
  $\mathbf{S}_m = \xx^m - \xx^{m+1}$ holds.

  By $L$-smoothness of the objective function $F$ (\Cref{assumption:global-function-nc}), we obtain the standard bound \citep{nesterov_lectures_2018}
  \begin{eqnarray*}
      F(\xx^{m+1}) & \le & F(\xx^m) + \abr{\nabla F(\xx^m), \xx^{m+1} - \xx^m} + \frac{L}{2} \norm{\xx^{m+1} - \xx^m}^2 \\
      & = & F(\xx^m) - \abr{\nabla F(\xx^m), \mathbf{S}_m} + \frac{L}{2} \norm{\mathbf{S}_m}^2 .
  \end{eqnarray*}
  Taking expectations conditional on information available at the beginning of cycle $m$, we find
  \begin{align}
    \label{eq:descent}
      \ExpSub{m}{F(\xx^{m+1})} & \le F(\xx^m) - \abr{\nabla F(\xx^m), \ExpSub{m}{\mathbf{S}_m}} + \frac{L}{2} \ExpSub{m}{\norm{\mathbf{S}_m}^2} \notag \\
      & \overset{\eqref{eq:cycle-step-decomposition}}{=} F(\xx^m) - \abr{\nabla F(\xx^m), \alpha \nabla F(\xx^m) + \ExpSub{m}{\mathbf{b}_m} + \ExpSub{m}{\boldsymbol{\nu}_m}} + \frac{L}{2} \ExpSub{m}{\norm{\mathbf{S}_m}^2} \notag \\
      & = F(\xx^m) - \alpha \norm{\nabla F(\xx^m)}^2 - \abr{\nabla F(\xx^m), \ExpSub{m}{\mathbf{b}_m}} + \frac{L}{2} \ExpSub{m}{\norm{\mathbf{S}_m}^2} ,
  \end{align}
  where we used \Cref{lemma:noise-term} to drop $\ExpSub{m}{\boldsymbol{\nu}_m} = 0$.
  We bound the inner product term
  \begin{align}
    \label{eq:aux-1}
      - \abr{\nabla F(\xx^m), \ExpSub{m}{\mathbf{b}_m}} & \overset{\eqref{eq:cauchy-schwarz}}{\le} \norm{\nabla F(\xx^m)} \norm{\ExpSub{m}{\mathbf{b}_m}} \notag \\
      & \overset{\eqref{eq:young-fancy}}{\le} \frac{\alpha}{4} \norm{\nabla F(\xx^m)}^2 + \frac{1}{\alpha} \norm{\ExpSub{m}{\mathbf{b}_m}}^2 \notag \\
      & \overset{\eqref{eq:jensen}}{\le} \frac{\alpha}{4} \norm{\nabla F(\xx^m)}^2 + \frac{1}{\alpha} \ExpSub{m}{\norm{\mathbf{b}_m}^2} .
  \end{align}

  Now,
  \begin{eqnarray*}
      \norm{\mathbf{S}_m}^2 & \overset{\eqref{eq:cycle-step-decomposition}}{=} & \norm{\alpha \nabla F(\xx^m) + \mathbf{b}_m + \boldsymbol{\nu}_m}^2 \\
      & \overset{\eqref{eq:squared-sum}}{\le} & 3 \alpha^2 \norm{\nabla F(\xx^m)}^2 + 3 \norm{\mathbf{b}_m}^2 + 3 \norm{\boldsymbol{\nu}_m}^2 ,
  \end{eqnarray*}

  and thus
  \begin{equation}
    \label{eq:aux-2}
      \ExpSub{m}{\norm{\mathbf{S}_m}^2} \le 3 \alpha^2 \norm{\nabla F(\xx^m)}^2 + 3 \ExpSub{m}{\norm{\mathbf{b}_m}^2} + 3 A \sigma^2 ,
  \end{equation}

  utilizing the second part of \Cref{lemma:noise-term}.
  
  Using \eqref{eq:aux-1} and \eqref{eq:aux-2} to further bound \eqref{eq:descent}, we obtain
  \begin{eqnarray*}
      \ExpSub{m}{F(\xx^{m+1})} & \overset{\eqref{eq:descent}}{\le} & F(\xx^m) - \alpha \norm{\nabla F(\xx^m)}^2 - \abr{\nabla F(\xx^m), \ExpSub{m}{\mathbf{b}_m}} + \frac{L}{2} \ExpSub{m}{\norm{\mathbf{S}_m}^2} \\
      & \overset{\eqref{eq:aux-1},\eqref{eq:aux-2}}{\le} & F(\xx^m) - \alpha \norm{\nabla F(\xx^m)}^2 \\
      && + \frac{\alpha}{4} \norm{\nabla F(\xx^m)}^2 + \frac{1}{\alpha} \norm{\ExpSub{m}{\mathbf{b}_m}}^2 \\
      && \quad + \frac{L}{2} \rbr{3 \alpha^2 \norm{\nabla F(\xx^m)}^2 + 3 \ExpSub{m}{\norm{\mathbf{b}_m}^2} + 3 A \sigma^2} \\
      & \overset{\eqref{eq:jensen}}{\le} & F(\xx^m) - \rbr{\frac{3}{4} \alpha - \frac{3 L}{2} \alpha^2} \norm{\nabla F(\xx^m)}^2 \\
      && + \frac{3 A L \sigma^2}{2} + \rbr{\frac{1}{\alpha} + \frac{3 L}{2}} \ExpSub{m}{\norm{\mathbf{b}_m}^2} .
  \end{eqnarray*}

  For $\alpha \le \frac{1}{6 L}$, we have $\rbr{\frac{3}{4} \alpha - \frac{3L}{2} \alpha^2} \ge \frac{\alpha}{2}$ and thus
  \begin{eqnarray*}
      \ExpSub{m}{F(\xx^{m+1})} & \le & F(\xx^m) - \frac{\alpha}{2} \norm{\nabla F(\xx^m)}^2 + \frac{3 A L \sigma^2}{2} + \rbr{\frac{1}{\alpha} + \frac{3L}{2}} \ExpSub{m}{\norm{\mathbf{b}_m}^2} .
  \end{eqnarray*}

  Rearranging gives
  \begin{eqnarray*}
      \norm{\nabla F(\xx^m)}^2 & \le & \frac{2 \rbr{F(\xx^m) - \ExpSub{m}{F(\xx^{m+1})}}}{\alpha} + \frac{3 A L \sigma^2}{\alpha} + \rbr{\frac{2}{\alpha^2} + \frac{3 L}{\alpha}} \ExpSub{m}{\norm{\mathbf{b}_m}^2} .
  \end{eqnarray*}

  Taking unconditional expectation, averaging, and applying the bound $F(\xx^0) - \Exp{F(\xx^{M})} \le F(\xx^0) - F^* = \Delta$ (\Cref{assumption:global-function-nc}) gives
  \begin{eqnarray*}
      \frac{1}{M} \sum_{m=0}^{M-1} \Exp{\norm{\nabla F(\xx^m)}^2} & \le & \frac{2 \sum_{m=0}^{M-1} \Exp{F(\xx^m) - F(\xx^{m+1})}}{\alpha M} \\
      && + \frac{3 A L \sigma^2}{\alpha} + \rbr{\frac{2}{\alpha^2} + \frac{3L}{\alpha}} \frac{1}{M} \sum_{m=0}^{M-1} \Exp{\ExpSub{m}{\norm{\mathbf{b}_m}^2}} \\
      & \le & \frac{2 \Delta}{\alpha M} + \frac{3 A L \sigma^2}{\alpha} + \rbr{\frac{2}{\alpha^2} + \frac{3L}{\alpha}} \frac{1}{M} \sum_{m=0}^{M-1} \Exp{\norm{\mathbf{b}_m}^2} .
  \end{eqnarray*}
\end{proof}

\begin{restatable}[Convergence to a Weighted Average]{theorem}{MasterBound}
  \label{thm:convergence-bound-general}
    Let the target objective be $F(\xx) = \sum_{i=1}^n w_i F_i(\xx)$, and suppose \Cref{assumption:computation-times,assumption:unbiasedness,assumption:bounded-gradient-variance,assumption:global-function-nc,assumption:local-f-smooth,assumption:bounded-heterogeneity} hold. 
    Assume the worker-specific stepsizes are chosen such that $\gamma_i \propto w_i \tau_i$. 
    If the maximum stepsize satisfies $0 < \gamma_{\max} \le \frac{1}{5 K L_{\max} \rho}$ 
    and the cycle stepsize $\alpha \coloneqq \sum_{k=0}^{K-1} \gamma_{i_k}$ satisfies $0 < \alpha \le \frac{1}{6 L}$, 
    then, after $M \ge 1$ cycles, the cycle iterates $\xx^m$ generated by \Cref{algo:asgd-iv} satisfy
    \begin{equation*}
        \frac{1}{M} \sum_{m=0}^{M-1} \Exp{\norm{\nabla F(\xx^m)}^2} \le c_0 \cdot \frac{\Delta}{\alpha M} + c_1 \cdot \frac{A}{\alpha} L \sigma^2 + c_2 \cdot \rbr{\frac{A}{\alpha}}^2 K^2 L_{\max}^2 \rbr{\sigma^2 + \zeta^2} ,
    \end{equation*}

    for some constants $c_0, c_1, c_2 > 0$, where $A \coloneqq \sum_{k=0}^{K-1} \gamma_{i_k}^2$.
\end{restatable}

\begin{proof}
    Plug the bound on the cycle bias from \Cref{lemma:bias-term} into \Cref{thm:near-stationarity-abstract-bias} to obtain
    \begin{eqnarray*}
        \frac{1}{M} \sum_{m=0}^{M-1} \Exp{\norm{\nabla F(\xx^m)}^2} & \le & \frac{2 \Delta}{\alpha M} \\
        && + \frac{3 A L \sigma^2}{\alpha} \\
        && \quad + \rbr{\frac{2}{\alpha^2} + \frac{3L}{\alpha}} 2 A^2 K^2 L_{\max}^2 \rbr{\sigma^2 + \zeta^2} \\
        && \qquad + \rbr{\frac{2}{\alpha^2} + \frac{3L}{\alpha}} 4 A^2 K^2 L_{\max}^2 \rho^2 \frac{1}{M} \sum_{m=0}^{M-1} \Exp{\norm{\nabla F(\xx^m)}^2} .
    \end{eqnarray*}

    Now, if $\alpha \le \frac{1}{6 L}$, then also $\rbr{\frac{2}{\alpha^2} + \frac{3L}{\alpha}} \le \frac{2.5}{\alpha^2}$.
    Grouping terms involving the gradient norm on the left gives
    \begin{eqnarray*}
        \rbr{1 - \frac{10 A^2 K^2 L_{\max}^2 \rho^2}{\alpha^2}} \frac{1}{M} \sum_{m=0}^{M-1} \Exp{\norm{\nabla F(\xx^m)}^2} & \le & \frac{2 \Delta}{\alpha M} + \frac{3 A L \sigma^2}{\alpha} + \frac{5 A^2 K^2 L_{\max}^2 \rbr{\sigma^2 + \zeta^2}}{\alpha^2}
    \end{eqnarray*}

    We require the first factor on the left-hand side to exceed $\nicefrac{1}{2}$, or, equivalently, $\frac{10 A^2 K^2 L_{\max}^2 \rho^2}{\alpha^2} \le \frac{1}{2}$.
    Noting that 
    $$
      A = \sum_{k=0}^{K-1} \gamma_{i_k}^2 \le \gamma_{\max} \sum_{k=0}^{K-1} \gamma_{i_k} = \gamma_{\max} \alpha,
    $$
    this requirement is met if $10 \gamma_{\max}^2 K^2 L_{\max}^2 \rho^2 \le \frac{1}{2}$, or, if 
    $$
      \gamma_{\max} \le \frac{1}{5 K L_{\max} \rho} \le \frac{1}{\sqrt{20} K L_{\max} \rho}.
    $$

    Thus, for $\gamma_{\max} \le \frac{1}{5 K L_{\max} \rho}$, we find
    \begin{equation}
      \label{eq:masterbound}
        \frac{1}{M} \sum_{m=0}^{M-1} \Exp{\norm{\nabla F(\xx^m)}^2} \le 4 \frac{\Delta}{\alpha M} + 6 \frac{A}{\alpha} L \sigma^2 + 10 \frac{A^2}{\alpha^2} K^2 L_{\max}^2 \rbr{\sigma^2 + \zeta^2} .
    \end{equation}
\end{proof}


\subsection{Proof of \Cref{thm:r-asgd-bound} and \Cref{corollary:r-asgd-time-complexity}} 
\label{sub:proof_of_thm:r-asgd-bound_and_corollary:r-asgd-time-complexity}

\RASGDIterationBound*

\begin{proof}
  This is a special case of \Cref{thm:convergence-bound-general} with stepsizes chosen according to \eqref{eq:learning-rates-equal}.

  The cycle stepsize now becomes
  \begin{equation*}
    \alpha = \sum_{k=0}^{K-1} \gamma_{i_k} = \sum_{i=1}^n \gamma_i \cdot K_i = \sum_{i=1}^n \gamma \tau_i \frac{\tau_H}{n \tau_{\max}} \cdot \frac{\tau_{\max}}{\tau_i} = \gamma \tau_H ,
  \end{equation*}
  the sum of squared stepsizes
  \begin{equation*}
    A = \sum_{k=0}^{K-1} \gamma_{i_k}^2
    = \sum_{i=1}^n \gamma_i^2 \cdot K_i
    = \sum_{i=1}^n \gamma^2 \tau_i^2 \frac{\tau_H^2}{n^2 \tau_{\max}^2} \cdot \frac{\tau_{\max}}{\tau_i}
    = \gamma^2 \frac{\tau_H^2 \tau_A}{n \tau_{\max}} 
    = \gamma^2 \frac{\tau_H \tau_A}{K} ~,
  \end{equation*}
  and the maximum stepsize
  \begin{equation*}
    \gamma_{\max} = \gamma \tau_{\max} \frac{\tau_H}{n \tau_{\max}}
    = \gamma \frac{\tau_H}{n}
    = \gamma \frac{\tau_{\max}}{K} ~,
  \end{equation*}

  where we used the identity $\nicefrac{K}{n} = \nicefrac{\tau_{\max}}{\tau_H}$ (\Cref{assumption:computation-times}).

  Therefore, $\frac{A}{\alpha} = \frac{\gamma \tau_A}{K}$.
  Plugging this into \eqref{eq:masterbound} gives
  \begin{equation*}
    \frac{1}{M} \sum_{m=0}^{M-1} \Exp{\norm{\nabla F(\xx^m)}^2} \le 4 \frac{\Delta}{\alpha M} + 6 \frac{\gamma \tau_A L \sigma^2}{K} + 10 \gamma^2 \tau_A^2 L_{\max}^2 \rbr{\sigma^2 + \zeta^2} .
  \end{equation*}

  The bound on the cycle stepsize translates to 
  \begin{equation*}
    \alpha \le \frac{1}{6 L} \Leftrightarrow \gamma \le \frac{1}{6 L \tau_H} ~,
  \end{equation*}
  and that on the maximum stepsize to
  \begin{equation*}
    \gamma_{\max} \le \frac{1}{5 K L_{\max} \rho} \Leftrightarrow \gamma \le \frac{1}{5 L_{\max} \rho \tau_{\max}} ~.
  \end{equation*}
\end{proof}

\RASGDTimeComplexity*

\begin{proof}
  Under the assumptions of \Cref{thm:r-asgd-bound}, we have
  \begin{eqnarray*}
        \frac{1}{M} \sum_{m=0}^{M-1} \Exp{\norm{\nabla F(\xx^m)}^2} & \le & \underbrace{c_0 \cdot \frac{\Delta}{\gamma \tau_H M}}_{T_0} + \underbrace{c_1 \cdot \frac{\gamma \tau_A L \sigma^2}{K}}_{T_1} + \underbrace{c_2 \cdot \gamma^2 \tau_A^2 L_{\max}^2 \rbr{\sigma^2 + \zeta^2}}_{T_2} .
    \end{eqnarray*}

  To achieve $T_1 \le \frac{\varepsilon}{4}$, we require
  $$
    \gamma \le \frac{1}{4 c_1} \frac{\varepsilon K}{L \sigma^2 \tau_A} ~.
  $$
  To achieve $T_2 \le \frac{\varepsilon}{4}$, we likewise require 
  $$
    \gamma \le \frac{1}{\sqrt{4 c_2}} \frac{\sqrt{\varepsilon}}{L_{\max} \tau_A \sqrt{\sigma^2 + \zeta^2}}.
  $$

  The largest $\gamma$ satisfying these constraints, as well as those in \Cref{thm:r-asgd-bound}, is
  \begin{equation*}
    \hat\gamma \coloneqq \min\nbr{\frac{1}{4 c_1} \frac{\varepsilon K}{L \sigma^2 \tau_A}, \frac{1}{\sqrt{4 c_2}} \frac{\sqrt{\varepsilon}}{L_{\max} \tau_A \sqrt{\sigma^2 + \zeta^2}}, \frac{1}{5 L_{\max} \rho \tau_{\max}}, \frac{1}{6 L \tau_H}} .
  \end{equation*}

  To achieve $T_0 \le \frac{\varepsilon}{2}$ and bound the right-hand side of \Cref{thm:r-asgd-bound} by $\varepsilon$, the cycle count must satisfy $M \ge \frac{2 c_0 \Delta}{\varepsilon \hat\gamma \tau_H}$.
  Ignoring constant factors, this cycle bound can be expressed as
  \begin{equation*}
      M \gtrsim \frac{\Delta L \sigma^2}{\varepsilon^2 K \tau_H} \tau_A + \frac{\Delta L_{\max} \sqrt{\sigma^2 + \zeta^2}}{\varepsilon^{1.5} \tau_H} \tau_A + \frac{\Delta L_{\max} \rho}{\varepsilon \tau_H} \tau_{\max} + \frac{\Delta L}{\varepsilon} .
  \end{equation*}

  To obtain the wall-clock time complexity, we multiply the required number of cycles by the wall-clock duration of a single cycle.
  Under \Cref{assumption:computation-times}, a cycle is completed after $\tau_{\max} = \frac{K \tau_H}{n}$ time units.
  Thus, after
  \begin{equation*}
    \mathcal{O}\rbr{\frac{\Delta L \sigma^2}{n \varepsilon^2} \tau_A + \frac{\Delta L_{\max} \sqrt{\sigma^2 + \zeta^2}}{\varepsilon^{1.5}} \frac{\tau_{\max}}{\tau_H} \tau_A + \frac{\Delta L_{\max} \rho}{\varepsilon} \frac{\tau_{\max}}{\tau_H} \tau_{\max} + \frac{\Delta L}{\varepsilon} \tau_{\max}}
  \end{equation*}

  time units, we achieve the target accuracy.
\end{proof}


\subsection{Proof of \Cref{thm:v-asgd-bound} and \Cref{corollary:v-asgd-time-complexity}} 
\label{sub:proof_of_thm:v-asgd-bound_and_corollary:v-asgd-time-complexity}

\VASGDIterationBound*

\begin{proof}
  This is a special case of \Cref{thm:convergence-bound-general} with equal stepsizes $\gamma_i = \nicefrac{\gamma}{K}$.

  The cycle stepsize now becomes
  \begin{equation*}
    \alpha = \sum_{k=0}^{K-1} \gamma_{i_k} = K \cdot \frac{\gamma}{K} = \gamma ,
  \end{equation*}
  the sum of squared stepsizes
  \begin{equation*}
    A = \sum_{k=0}^{K-1} \gamma_{i_k}^2
    = K \cdot \rbr{\frac{\gamma}{K}}^2 = \frac{\gamma^2}{K},
  \end{equation*}
  and the maximum stepsize
  \begin{equation*}
    \gamma_{\max} = \frac{\gamma}{K} .
  \end{equation*}

  Therefore, $\frac{A}{\alpha} = \frac{\gamma}{K}$.
  Plugging this into \eqref{eq:masterbound} gives
  \begin{equation*}
    \frac{1}{M} \sum_{m=0}^{M-1} \Exp{\norm{\nabla \tilde{F}(\xx^m)}^2} \le 4 \frac{\tilde\Delta}{\gamma M} + 6 \frac{\gamma \tilde{L} \sigma^2}{K} + 10 \gamma^2 L_{\max}^2 \rbr{\sigma^2 + \tilde\zeta^2} .
  \end{equation*}

  The bound on the cycle stepsize translates to 
  \begin{equation*}
    \alpha \le \frac{1}{6 \tilde{L}} \Leftrightarrow \gamma \le \frac{1}{6 \tilde{L}},
  \end{equation*}
  and that on the maximum stepsize
  \begin{equation*}
    \gamma_{\max} \le \frac{1}{5 K L_{\max} \tilde\rho} \Leftrightarrow \gamma \le \frac{1}{5 L_{\max} \tilde\rho} .
  \end{equation*}
\end{proof}

\VASGDTimeComplexity*

\begin{proof}
  Under the assumptions of \Cref{thm:v-asgd-bound}, we have
  \begin{eqnarray*}
        \frac{1}{M} \sum_{m=0}^{M-1} \Exp{\norm{\nabla \tilde{F}(\xx^m)}^2} & \le & \underbrace{c_0 \cdot \frac{\tilde\Delta}{\gamma M}}_{T_0} + \underbrace{c_1 \cdot \frac{\gamma \tilde{L} \sigma^2}{K}}_{T_1} + \underbrace{c_2 \cdot \gamma^2 L_{\max}^2 \rbr{\sigma^2 + \tilde\zeta^2}}_{T_2} .
    \end{eqnarray*}

  To achieve $T_1 \le \frac{\varepsilon}{4}$, we require $\gamma \le \frac{1}{4 c_1} \frac{\varepsilon K}{\tilde{L} \sigma^2}$.
  To achieve $T_2 \le \frac{\varepsilon}{4}$, we likewise require $\gamma \le \frac{1}{\sqrt{4 c_2}} \frac{\sqrt{\varepsilon}}{L_{\max} \sqrt{\sigma^2 + \tilde\zeta^2}}$.

  The largest $\gamma$ satisfying these constraints, as well as those in \Cref{thm:v-asgd-bound}, is
  \begin{equation*}
    \hat\gamma \coloneqq \min\nbr{\frac{1}{4 c_1} \frac{K \varepsilon}{\tilde{L} \sigma^2}, \frac{1}{\sqrt{4 c_2}} \frac{\sqrt{\varepsilon}}{L_{\max} \sqrt{\sigma^2 + \tilde\zeta^2}}, \frac{1}{5 L_{\max} \tilde\rho}, \frac{1}{6 \tilde{L}}} .
  \end{equation*}

  To achieve $T_0 \le \frac{\varepsilon}{2}$ and bound the right-hand side of \Cref{thm:v-asgd-bound} by $\varepsilon$, the cycle count must satisfy $M \ge \frac{2 c_0 \tilde\Delta}{\varepsilon \hat\gamma}$.
  Ignoring constant factors, this cycle bound can be expressed as
  \begin{equation*}
      M \gtrsim \frac{\tilde\Delta \tilde{L} \sigma^2}{\varepsilon^2 K} + \frac{\tilde\Delta L_{\max} \sqrt{\sigma^2 + \tilde\zeta^2}}{\varepsilon^{1.5}} + \frac{\tilde\Delta L_{\max} \tilde\rho}{\varepsilon} + \frac{\tilde\Delta \tilde{L}}{\varepsilon} .
  \end{equation*}

  To obtain the wall-clock time complexity, we multiply the required number of cycles by the wall-clock duration of a single cycle.
  Under \Cref{assumption:computation-times}, a cycle is completed after $\tau_{\max} = \frac{K \tau_H}{n}$ time units.
  Thus, after
  \begin{equation*}
    \mathcal{O}\rbr{\frac{\tilde\Delta \tilde{L} \sigma^2}{n \varepsilon^2} \tau_H + \frac{\tilde\Delta L_{\max} \sqrt{\sigma^2 + \tilde\zeta^2}}{\varepsilon^{1.5}} \tau_{\max} + \frac{\tilde\Delta L_{\max} \tilde\rho}{\varepsilon} \tau_{\max} + \frac{\tilde\Delta \tilde{L}}{\varepsilon} \tau_{\max}}
  \end{equation*}

  time units, we achieve the target accuracy.
\end{proof}



\section{On the Necessity of Bounding Data Heterogeneity}
\label{sec:on_the_bounded_data_heterogeneity_assumption}

In this section, we present a simple counterexample to demonstrate that without a bounded heterogeneity assumption (such as \Cref{assumption:bounded-heterogeneity}), \rasgd with worker-wise constant stepsizes cannot guarantee convergence to an $\varepsilon$-stationary point.

Consider a two-worker setup with local objective functions
\begin{equation*}
    F_1(x) = \frac{x^2}{2} - cx, \quad F_2(x) = \frac{x^2}{2} + cx ~.
\end{equation*}

The global objective is the equal-weighted average $F(x) = \frac{x^2}{2}$, with $\nabla F(x) = x$ and a unique stationary point at $x^* = 0$.
The constant $c > 0$ captures the degree of data heterogeneity, $c = \norm{\nabla F_i(x) - \nabla F(x)}$.

Suppose $\tau_1 = 1, \tau_2 = 2$, so worker 1 takes two steps of size $\gamma$ per cycle while worker 2 takes only one of size $2 \gamma$.
Assume the cycle begins at a point $\abs{x_0} < 1$ so that $\norm{\nabla F(x_0)} < 1$.
First, worker 1 delivers the gradient computed at the initial model, $\nabla F_1(x_0) = x_0 - c$, and the server updates the model:
\begin{equation*}
    x_1^0 = x_0 - \gamma(x_0 - c) = x_0 (1 - \gamma) + \gamma c ~.
\end{equation*}

Worker 1 delivers their second update, $\nabla F_1(x_1^0) = x_1^0 - c$, and the server updates the model again:
\begin{equation*}
    x_2^0 = x_1^0 - \gamma(x_1^0 - c) = x_1^0(1-\gamma) + \gamma c = x_0(1-\gamma)^2 + c(2\gamma - \gamma^2) ~.
\end{equation*}

Worker 2 delivers their stale gradient, $\nabla F_2(x_0) = x_0 + c$, and the server updates the model:
\begin{equation*}
    x^1 = x_3^0 = x_2^0 - 2 \gamma (x_0 + c) = x_0(1 - 4\gamma + \gamma^2) - \gamma^2 c ~.
\end{equation*}

Thus, for any fixed $\gamma > 0$, the gradient norm at the end of the cycle can be arbitrarily large if there is no restrictions on $c$.
In particular, if 
\begin{equation*}
  c > \frac{1 + \left| x_0 (1 - 4\gamma + \gamma^2) \right|}{\gamma^2} ~,
\end{equation*}

then $\norm{\nabla F(x^1)} > 1$.

\paragraph{Relaxing the Assumption} 
\label{par:relaxing_the_assumption}

\citet{koloskova_sharper_2022,mishchenko_asynchronous_2022} impose the constraint 
\begin{equation}
  \label{eq:bounded-heterogeneity-simple}
  \norm{\nabla F_i(\xx) - \nabla F(\xx)}^2 \le \xi^2 , \quad \forall \xx \in \R^d, \quad \forall i ~.
\end{equation}

This assumption implies our \Cref{assumption:bounded-heterogeneity} and is therefore more restrictive.
Suppose \eqref{eq:bounded-heterogeneity-simple} holds, then
\begin{eqnarray*}
  \norm{\nabla F_i(\xx)}^2 & \overset{\eqref{eq:squared-sum}}{\le} & 2 \norm{\nabla F_i(\xx) - \nabla F(\xx)}^2 + 2 \norm{\nabla F(\xx)}^2 \\
  & \overset{\eqref{eq:bounded-heterogeneity-simple}}{\le} & 2 \xi^2 + 2 \norm{\nabla F(\xx)}^2 ,
\end{eqnarray*}

so \Cref{assumption:bounded-heterogeneity} also holds for $\zeta^2 = 2 \xi^2$ and $\rho^2 = 2$.

Note that \eqref{eq:bounded-heterogeneity-simple} is quite restrictive as it does not even allow for quadratics with different Hessians and smoothness constants.
To illustrate, consider two simple quadratics, $F_1(x) = \frac{a}{2} x^2$ and $F_2(x) = \frac{b}{2} x^2$. 
The deviation for the first worker is $\norm{\nabla F_1(x) - \nabla F(x)} = | \frac{a-b}{2} | |x|$. 
For this term to remain bounded by a constant $\xi^2$ for all $x \in \R$, we must have $a=b$. 

Our relaxed \Cref{assumption:bounded-heterogeneity} accommodates such cases by allowing the local gradient bounds to scale proportionally with the global gradient norm.


\section{Wall-Clock Time Complexities} 
\label{sec:wall_clock_time_complexities_of_concurrent_and_delay_adaptive_asgd}

\subsection{Concurrent ASGD} 
\label{sub:concurrent_asgd}

\citet{koloskova_sharper_2022} give a bound on the iteration complexity to achieve $\varepsilon$-stationarity.
The leading term of this bound is $\frac{\Delta L (\sigma^2 + \zeta^2)}{\varepsilon^2}$, where $\zeta^2$ is a constant bounding data heterogeneity (cf. \Cref{sec:on_the_bounded_data_heterogeneity_assumption}).

Because \koloskovaasgd employs uniform worker-sampling and keeps the number of gradients in-flight constant throughout, each worker will contribute $\nicefrac{M}{n}$ gradients in the long run, where $M$ denotes the number of iterations.
The time needed to compute these gradients is bottlenecked by the slowest worker and thus scales with the maximum of the computation times, $\tau_{\max}$, not their arithmetic mean, $\tau_A$.
In other words, the expected time to complete a large number $M$ of iterations of \koloskovaasgd will be performed in $\nicefrac{M \tau_{\max}}{n}$ units of time.
The leading term of the time complexity bound is therefore $\frac{\Delta L (\sigma^2 + \zeta^2)}{n \varepsilon^2} \tau_{\max}$ as stated in \Cref{table:complexities}.


\subsection{Delay-Adaptive ASGD} 
\label{sub:delay_adaptive_asgd}

In \mishchenkoasgd, stepsizes are scaled down in proportion to the gradient staleness.
In our model with a cyclic update schedule (cf. \Cref{sec:asyncsgd_with_a_cyclic_update_schedule}), these stepsizes are therefore deterministic and do not change per within-cycle iteration after the second cycle.
Let $\delta_k$ denote the number of model updates applied between worker $i_k$ receiving the model and the server applying the gradient in iteration $k$.
Then the stepsize applied is $\gamma_k = \frac{\gamma}{1 + \delta_k}$.
The total aggregate stepsize of steps taken by worker $i$ over one cycle is then
\begin{equation}
    \label{eq:cap-gamma}
    \Gamma_i \coloneqq \sum_{k : i_k = i} \gamma_k = \sum_{k : i_k = i} \frac{\gamma}{1 + \delta_k} ,
\end{equation}

and we can decompose the cycle step as before (cf. \Cref{lemma:cycle-step-decomposition}).
The objective function targeted by \mishchenkoasgd is a weighted-average, $\hat{F}(\xx) = \sum_{i=1}^n \hat{w}_i F_i(\xx)$, with weights $\hat{w}_i \propto \Gamma_i$.
If the local objective functions are linearly independent, then $\hat{F}$ coincides with the equal-weighted average $F$ only if all workers have the same computation time, $\tau_i = \tau_A$.
In this case, \mishchenkoasgd coincides with \rasgd and \vasgd.

In the general case, slow workers are penalized twice---due to the lower number of updates per cycle and the smaller stepsizes.
By letting faster workers take larger steps, \mishchenkoasgd may converge faster, albeit to the wrong objective function $\hat{F}$.

Under \Cref{assumption:computation-times}, if simultaneous updates are processed in ascending order of worker indices, the iteration-specific delay $\delta_k$ for every update delivered by worker $i$ remains constant across all iterations $k$ with $i_k = i$.
With a slight abuse of notation, the worker-specific delay is then 
\begin{equation}
    \label{eq:delta-i}
    \delta_i = \sum_{j \neq i} \frac{\tau_i}{\tau_j} = \tau_i \frac{n}{\tau_H} - 1 .
\end{equation}

Substituting \eqref{eq:delta-i} into \eqref{eq:cap-gamma}, we find $\Gamma_i = \gamma \frac{\tau_{\max} \tau_H}{n} \frac{1}{\tau_i^2}$.
With this, the cycle stepsize and sum of squared stepsizes become
\begin{equation*}
    \alpha = \gamma \frac{\tau_{\max} \tau_H}{n} \sum_{i=1}^n \tau_i^{-2} , \quad A = \gamma^2 \frac{\tau_{\max} \tau_H^2}{n^2} \sum_{i=1}^n \tau_i^{-3} .
\end{equation*}

Following the same steps as in \Cref{sub:proof_of_thm:r-asgd-bound_and_corollary:r-asgd-time-complexity,sub:proof_of_thm:v-asgd-bound_and_corollary:v-asgd-time-complexity}, we can now derive the cycle and wall-clock time complexities.
Under the same assumptions as in \Cref{thm:convergence-bound-general} (imposed on $\hat{F}$), the leading term in the cycle complexity becomes
\begin{equation*}
    \frac{\hat\Delta \hat{L} \sigma^2}{\varepsilon^2 \tau_{\max}} \frac{\sum_{i=1}^n \tau_i^{-3}}{\rbr{\sum_{i=1}^n \tau_i^{-2}}^2} ,
\end{equation*}

and that in the time complexity bound
\begin{equation*}
    \frac{\hat\Delta \hat{L} \sigma^2}{\varepsilon^2} \frac{\sum_{i=1}^n \tau_i^{-3}}{\rbr{\sum_{i=1}^n \tau_i^{-2}}^2} = \frac{\hat\Delta \hat{L} \sigma^2}{n \varepsilon^2} \tau_{DA} ,
\end{equation*}

where $\tau_{DA} \coloneqq n \frac{\sum_{i=1}^n \tau_i^{-3}}{\rbr{\sum_{i=1}^n \tau_i^{-2}}^2}$.
That is, the leading term of the time complexity of \mishchenkoasgd scales with $\tau_{DA}$, which collapses to the arithmetic mean if all workers have the same computation speed.
We can bound
\begin{equation*}
    \tau_H \le \tau_{DA} \le n \tau_{\min} ,
\end{equation*}

which shows that \mishchenkoasgd may converge faster than \rasgd---e.g., if there is one fast worker with $\tau_1 = 1$ and one slow worker with $\tau_2 = 10$, so $\tau_{DA} \approx 1.96 < 5.5 < \tau_A$---or slower---e.g., under near-homogeneity with a slight skew, $\tau_1 = 1$, $\tau_2 = \dots = \tau_{10}$, so $\tau_{DA} \approx 2.01 > 1.9 = \tau_A$---but the objective function targeted in such cases is, in general, not the equal-weighted average \eqref{eq:objective-intro}.

The leading term scaling with $\tau_A$ reported in \Cref{table:complexities} refers to the setting with homogeneous worker computation times under which \mishchenkoasgd can guarantee convergence to an $\varepsilon$-stationary point for arbitrarily small $\varepsilon$.



\section{Experimental Details} 
\label{sec:experimental_details}

For our experiments in \Cref{sec:experiments}, we consider an image classification problem on MNIST \citep{lecun_mnist_2010} with standard normalization.
To enforce a high degree of data heterogeneity, we assign each one of the ten classes to one of ten workers.
This is a special case of Dirichlet partitioning \citep{hsu_measuring_2019,yurochkin_bayesian_2019} with concentration parameter set to 0.
We trim the workers' data to ensure equal sample sizes, resulting in $5,421$ examples per client.

Our model is a two-layer MLP with ReLU activations in the hidden layer of dimension 128.
The workers compute stochastic gradients using minibatches of their local data of size 64.
The model is trained by minimizing cross-entropy loss.

We assign $\tau_i \in \{1,2,4,8,16\}$ to two workers each, making the fastest workers 16 times faster than the slowest.
In the fixed-computation setting, worker $i$ takes $\tau_i$ units of time to compute a stochastic gradient.
In the fluctuation computation-time setting, worker $i$'s computation times are sampled from an exponential distribution with mean $\tau_i$ instead.

For each method, we tune the stepsize parameter $\gamma$ within a fixed wall-clock budget of $30,000$ time units to minimize loss, and sweep the set $\{10^x: x=-4,\dots,0\}$.
The selected stepsizes all fall in the interior of this grid.

We report the training loss versus wall-clock time in \Cref{fig:loss} for the three methods under consideration.
Each method is run five times over different seeds.
The solid lines show the median over these runs and the shaded region corresponds to the minimum and maximum.
To enable these computations when computation times fluctuate, the loss trajectories have been linearly interpolated over a grid of size 200 covering the simulated time horizon.

We find that \rasgd outperforms \malenia and \ringleader in this setup.
Although our theoretical analysis suggests all three methods to have similar performance in the fixed-computation model, it appears that \rasgd benefits from the larger number of model updates.
This is consistent with the finding that, in the setting with fluctuating computation-times, the performance of \rasgd is not affected as it takes roughly the same number as updates as before, whereas \malenia and \ringleader are slowed down due to, in expectation, longer gathering phases.


\section{Further Experiments} 
\label{sec:further_experiments}

\subsection{Vanilla ASGD Targets the Frequency-Weighted Average} 
\label{sub:objective_inconsistency_of_vanilla_asgd}

To illustrate that \rasgd targets the right objective function \eqref{eq:objective-intro}, while \vasgd instead targets the frequency-weighted average of the local objective \eqref{eq:f-tilde}, we consider a simple example with two workers with computation times $\tau_1 = 1, \tau_2 = 2$, and local objective functions $F_1(x) = (x-4)^2$ and $F_2(x) = 2(x+3)^2$.
The unique minimizer of the equal-weighted average is $x^* = -\nicefrac{2}{3}$, while that of the frequency-weighted average is $\tilde x^* = +\nicefrac{1}{2}$.

We run both algorithms with the same cycle stepsize $\alpha = 0.01$, setting gradient stochasticity to zero to isolate the effects of staleness.
Their trajectories in the search space are shown in \Cref{fig:vanilla}.
Panel (a) shows the full sequence of iterates, including the within-cycle iterates, while panel (b) shows the cycle iterates only.
We see that \rasgd converges to the minimizer of the equal-weighted average $F$, while \vasgd converges to the minimizer of the frequency-weighted average $\tilde F$.
Moreover, \vasgd converges faster for the same cycle stepsize $\alpha$, consistent with our discussion of \Cref{corollary:v-asgd-time-complexity} where we have seen that the leading term scales with the harmonic mean of workers' computation times.
Note that the oscillations seen in the trajectories covering all iterates in panel (a) are absent when shifting to the cycle-iterate trajectories in (b).

Lastly, a close look at the trajectory of \vasgd in panel (b) shows that the convergence is not exactly to the stationary point.
This, too, is consistent with our theory:
Even in the absence of gradient stochasticity ($\sigma^2 = 0$), a neighborhood term reflecting the staleness bias persists.
As seen in \Cref{thm:r-asgd-bound,thm:v-asgd-bound}, this scales with $\gamma^2$ and can thus be efficiently controlled by choosing a smaller stepsize.

\begin{figure}[htbp]
  \centering
  \begin{subfigure}[b]{0.49\textwidth}
    \centering
    \includegraphics[width=\textwidth]{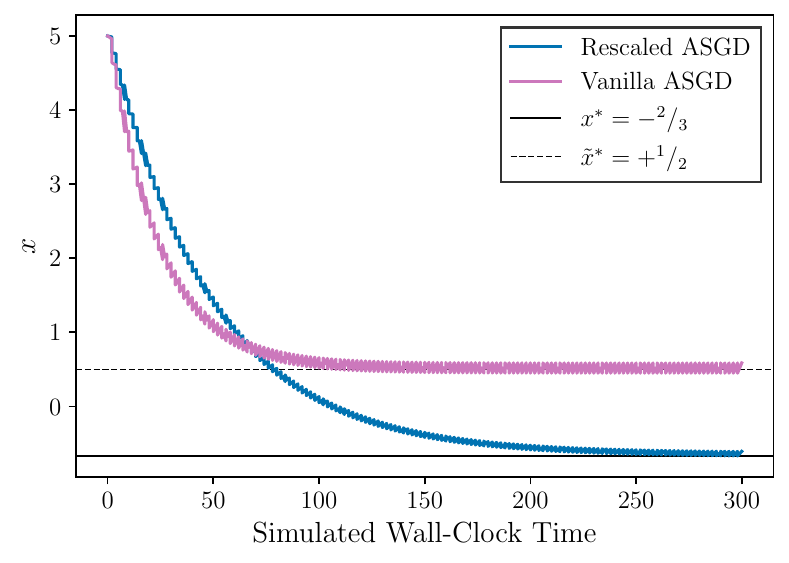}
    \caption{All iterates $x_k^m$.}
    \label{fig:1d-ex-all}
  \end{subfigure}
  \hfill
  \begin{subfigure}[b]{0.49\textwidth}
    \centering
    \includegraphics[width=\textwidth]{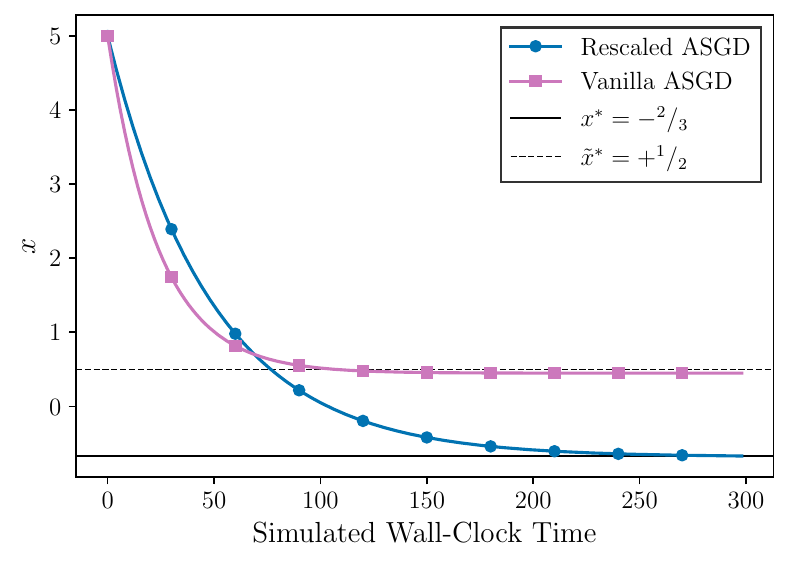}
    \caption{Cycle iterates $x^m$.}
    \label{fig:1d-ex-cycle}
  \end{subfigure}
  \caption{\rasgd targets the equal-weighted average $F$, \vasgd the frequency-weighted average $\tilde F$.}
  \label{fig:vanilla}
\end{figure}


\subsection{Delay-Adaptive SGD Exhibits Objective Inconsistency} 
\label{sub:why_delay_adaptive_sgd_struggles_with_data_heterogeneity}

\mishchenkoasgd \citep{mishchenko_asynchronous_2022} scales stepsizes down in proportion to the gradient staleness, i.e., the number of iterations that have passed between the worker receiving the model from the server and them delivering the gradient back.
In the fixed-computation model, this approach does the opposite of \rasgd.

To illustrate, consider another two-worker setting with one fast worker ($\tau_1 = 1$) and one slow worker ($\tau_2 = 100$).
99 out of 100 gradients delivered by the fast worker will exhibit no staleness, so \mishchenkoasgd does not scale the stepsizes down.
Merely one gradient of the fast worker in each cycle will be shrunk by a factor of two due to the slow worker delivering theirs.
The one gradient delivered by the slow worker, in contrast, will always be stale by 100 iterations, and thus the stepsize scaled down accordingly.

The search trajectory of \mishchenkoasgd will therefore be heavily biased towards the local objective function of the fast worker, even more so than under \vasgd.

\Cref{fig:adaptive} shows a simulation for $F_1(x) = (x-1)^2, F_2(x) = (x+1)^2$.
As expected, \mishchenkoasgd converges to a point close to $x_1^* = 1$, the minimizer of worker 1's objective function, whereas \rasgd converges to a small neighborhood around the minimizer, $x^* = 0$, of the equal-weighted average.

\begin{figure}[htbp]
  \centering
  \includegraphics[width=0.8\textwidth]{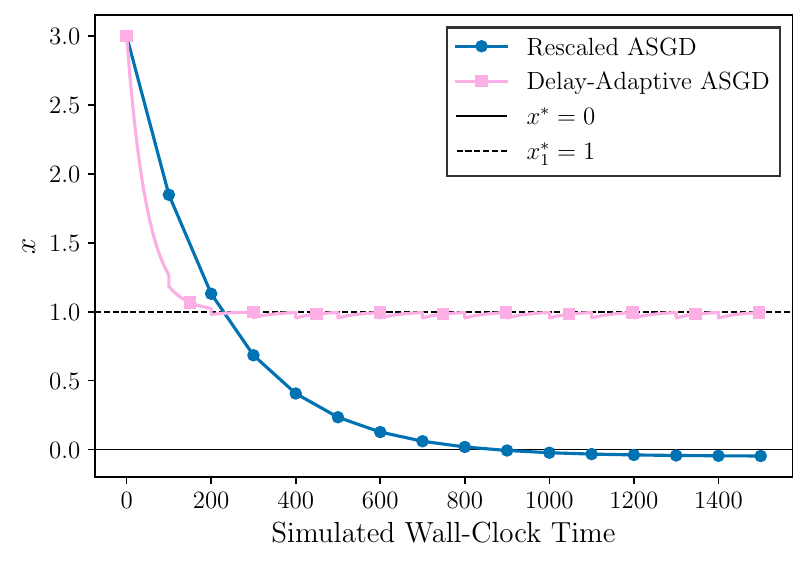}
  \caption{\mishchenkoasgd is heavily biased towards the local objective of the faster worker. \rasgd converges to a small neighborhood around the minimizer of the equal-weighted average.}
  \label{fig:adaptive}
\end{figure}



\end{document}

%% file: macros.tex
\usepackage{graphicx}
\usepackage{apptools}
\usepackage[flushleft]{threeparttable}
\usepackage{array,booktabs,makecell}
\usepackage{multirow}
\usepackage{nicefrac}
\usepackage{amsthm}
\usepackage{mathtools}
\usepackage{amsmath,amsfonts,bm,amssymb}
\usepackage{dsfont}
\usepackage{wrapfig}
\usepackage{caption}
\usepackage{siunitx}
\usepackage{thm-restate}
\usepackage{nccmath}
\usepackage{empheq}
\usepackage{bbm}
\usepackage{tabularx}

\usepackage{algorithmic}
\usepackage{algorithm}
\usepackage{filecontents}

\usepackage{color}
\usepackage{colortbl}
\definecolor{bgcolor}{rgb}{0.76,0.88,0.50}
\definecolor{bgcolor0}{rgb}{0.93,0.99,1}
\definecolor{bgcolor1}{rgb}{0.8,1,1}
\definecolor{bgcolor2}{rgb}{0.8,1,0.8}
\definecolor{bgcolor3}{rgb}{0.50,0.90,0.50}
\usepackage{tcolorbox}
\usepackage{pifont}

\definecolor{mydarkblue}{rgb}{0,0.16,0.54}
\definecolor{mydarkgreen}{RGB}{39,130,67}
\definecolor{mydarkorange}{RGB}{236,147,14}
\definecolor{mydarkred}{RGB}{192,47,25}
\definecolor{mydarkgreen}{RGB}{0,100,0}
\definecolor{ruby}{RGB}{155,17,30}
\definecolor{chili}{RGB}{191,0,0}
\definecolor{sangria}{RGB}{146,0,10}
\definecolor{burgundy}{RGB}{128,0,32}
\definecolor{darkred}{RGB}{132,0,0} 
\definecolor{cherry}{RGB}{192,0,0} 
\definecolor{myaccent}{rgb}{0,0.45,0.45}
\definecolor{anotherdarkred}{rgb}{0.45,0,0.05}

\definecolor{linkcolor}{rgb}{0.45,0,0.05} 

\newcommand{\crossmarkred}{{\color{mydarkred}\ding{56}}}
\newcommand{\checkmarkgreen}{\textbf{\color{mydarkgreen}\ding{52}}}

\definecolor{pastelred}{HTML}{FFB3BA}
\definecolor{pastelgreen}{HTML}{BDECB6}
\definecolor{pastelblue}{HTML}{AEC6CF}

\usepackage{todonotes}

\usepackage{xspace}
\usepackage{relsize}

\providecommand{\ee}{\mathbf{e}}
\providecommand{\uu}{\mathbf{u}}
\providecommand{\vv}{\mathbf{v}}
\providecommand{\xx}{\mathbf{x}}
\providecommand{\yy}{\mathbf{y}}

\newcommand{\norm}[1]{\left\| #1 \right\|}
\newcommand{\sqnorm}[1]{\left\| #1 \right\|^2}

\newcommand{\abs}[1]{\left| #1 \right|}

\newcommand{\nbr}[1]{\left\{#1\right\}} 
\newcommand{\rbr}[1]{\left(#1\right)} 
\newcommand{\abr}[1]{\left\langle#1\right\rangle} 

\newcommand{\R}{\mathbb{R}} 
\newcommand{\N}{\mathbb{N}} 

\newcommand{\Exp}[1]{{\mathbb{E}}\left[#1\right]}
\newcommand{\ExpSub}[2]{{\mathbb{E}}_{#1}\left[#2\right]}

\newcommand{\cD}{\mathcal{D}}

\newcommand{\cF}{\mathcal{F}}

\newcommand{\cI}{\mathcal{I}}

\newcommand{\cO}{\mathcal{O}}

\makeatletter
\newcommand{\vast}{\bBigg@{4}}

\def\<{\left\langle}
\def\>{\right\rangle}
\def\[{\left[}
\def\]{\right]}
\def\({\left(}
\def\){\right)}

\definecolor{myblue}{rgb}{0.3,0.25,0.2}
\definecolor{myorange}{rgb}{0.3,0.25,0.2}

\definecolor{alggray}{gray}{0.35} 
\newcommand{\algname}[1]{\text{\color{alggray}\sffamily\relscale{0.93}#1}\xspace}

\newcommand{\ringleader}{\algname{Ringleader~ASGD}}

\newcommand{\malenia}{\algname{Malenia~SGD}}

\newcommand{\naiveminibatch}{\algname{Naive~Minibatch~SGD}}

\newcommand{\asyncsgdtitle}{Asynchronous~SGD}
\newcommand{\asyncsgd}{\algname{\asyncsgdtitle}}

\newcommand{\sgd}{\algname{SGD}}

\usepackage{thmtools}
\usepackage{thm-restate}
\usepackage{mdframed}

\theoremstyle{plain}
\newtheorem{theorem}{Theorem}[section]

\newtheorem{lemma}[theorem]{Lemma}

\newtheorem{assumption}{Assumption}[section]
\theoremstyle{definition}
\newtheorem{definition}{Definition}[section]
\theoremstyle{remark}

\theoremstyle{plain} 

\newmdenv[
  font=\normalfont\bfseries,
  linecolor=black,
  linewidth=0.8pt,
  topline=false,
  bottomline=false,
  leftline=false,
  rightline=false,
  backgroundcolor=gray!13,
  skipabove=2pt,
  skipbelow=2pt,
  innertopmargin=-5pt,
  innerbottommargin=4pt,
  innerrightmargin=4pt,
  innerleftmargin=4pt,
]{myboxed}

\newmdtheoremenv[
  font=\normalfont\bfseries,
  linecolor=black,
  linewidth=0.8pt,
  topline=false,
  bottomline=false,
  leftline=false,
  rightline=false,
  backgroundcolor=gray!13,
  skipabove=2pt,
  skipbelow=2pt,
  innertopmargin=-5pt,
  innerbottommargin=4pt,
  innerrightmargin=4pt,
  innerleftmargin=4pt,
]{boxedtheorem}[theorem]{Theorem}

\newmdtheoremenv[
  font=\normalfont\bfseries,
  linecolor=black,
  linewidth=0.8pt,
  topline=false,
  bottomline=false,
  leftline=false,
  rightline=false,
  backgroundcolor=gray!13,
  skipabove=2pt,
  skipbelow=2pt,
  innertopmargin=-5pt,
  innerbottommargin=4pt,
  innerrightmargin=4pt,
  innerleftmargin=4pt,
]{boxedlemma}[theorem]{Lemma}

\newmdtheoremenv[
  font=\normalfont\bfseries,
  linecolor=black,
  linewidth=0.8pt,
  topline=false,
  bottomline=false,
  leftline=false,
  rightline=false,
  backgroundcolor=gray!13,
  skipabove=2pt,
  skipbelow=2pt,
  innertopmargin=-5pt,
  innerbottommargin=4pt,
  innerrightmargin=4pt,
  innerleftmargin=4pt,
]{boxeddefinition}[theorem]{Definition}

\newmdtheoremenv[
  font=\normalfont\bfseries,
  linecolor=black,
  linewidth=0.8pt,
  topline=false,
  bottomline=false,
  leftline=false,
  rightline=false,
  backgroundcolor=gray!13,
  skipabove=2pt,
  skipbelow=2pt,
  innertopmargin=-5pt,
  innerbottommargin=4pt,
  innerrightmargin=4pt,
  innerleftmargin=4pt,
]{boxedassumption}[theorem]{Assumption}

\newtheorem{innercustomthm}{Theorem}
\newenvironment{restate-theorem}[1]
  {\renewcommand\theinnercustomthm{#1}\innercustomthm}
  {\endinnercustomthm}

\newtheorem{innercustomlemma}{Lemma}
\newenvironment{restate-lemma}[1]
  {\renewcommand\theinnercustomlemma{#1}\innercustomlemma}
  {\endinnercustomlemma}

\newtheorem{innercustomproposition}{Proposition}
\newenvironment{restate-proposition}[1]
  {\renewcommand\theinnercustomproposition{#1}\innercustomproposition}
  {\endinnercustomproposition}

\newenvironment{restate-boxedtheorem}[1]
  {\renewcommand\theinnercustomthm{#1}\begin{myboxed}\begin{innercustomthm}}
  {\end{innercustomthm}\end{myboxed}}

\newenvironment{restate-boxedlemma}[1]
  {\renewcommand\theinnercustomlemma{#1}\begin{myboxed}\begin{innercustomlemma}}
  {\end{innercustomlemma}\end{myboxed}}

\newcommand*{\sketchproofname}{Sketch of Proof}